\documentclass[twoside,11pt]{article}

\usepackage[preprint]{jmlr2e_perso}

\usepackage[colorlinks=true,linkcolor=blue,citecolor=blue]{hyperref}

\usepackage{enumerate}
\usepackage{csquotes}

\usepackage{subcaption}

\usepackage{mathtools}
\usepackage{calrsfs}
\DeclareMathAlphabet{\pazocal}{OMS}{zplm}{m}{n}

\usepackage{bm}

\usepackage{algorithm}
\usepackage{algpseudocode}

\usepackage[dvipsnames]{xcolor}

\usepackage{todonotes}
\definecolor{PalePurp}{rgb}{0.66,0.57,0.66}

\algrenewcommand\algorithmicensure{\textbf{Output:}}

\def\argmax{\mathop\text{argmax}}
\def\argmin{\mathop\text{argmin}}
\def\maximize{\mathop\text{maximize\,}}
\def\minimize{\mathop\text{minimize\,}}

\ShortHeadings{Robustness and risk management via distributional dynamic programming}{Achab and Neu}
\firstpageno{1}

\begin{document}

\title{Robustness and risk management via distributional dynamic programming}

\author{\name Mastane Achab \email mastane.achab@gmail.com \\
       \addr Universitat Pompeu Fabra\\
       Barcelona, Spain
       \AND
       \name Gergely Neu \email gergely.neu@gmail.com \\
       \addr Universitat Pompeu Fabra\\
       Barcelona, Spain}


\maketitle

\begin{abstract}
In dynamic programming (DP) and reinforcement learning (RL), an agent learns to act optimally in terms of expected long-term return by sequentially interacting with its environment modeled by a Markov decision process (MDP).
More generally in distributional reinforcement learning (DRL), the focus is on the whole distribution of the return, not just its expectation.
Although DRL-based methods produced state-of-the-art performance in RL with function approximation, they involve additional quantities (compared to the non-distributional setting) that are still not well understood.
As a first contribution, we introduce a new class of distributional operators, together with a practical DP algorithm for policy evaluation, that come
with a \emph{robust MDP} interpretation.
Indeed, our approach reformulates through an augmented state space where each state is split into a \emph{worst-case
substate} and a \emph{best-case substate}, whose values are maximized by \emph{safe} and \emph{risky} policies
respectively.
Finally, we derive distributional operators and DP algorithms solving a new control task:
How to distinguish safe from risky optimal actions in order to break ties in the space of optimal policies?
\end{abstract}


\section{Introduction}

This paper is concerned with \emph{robust} sequential decision making in
an uncertain environment, modeled by a Markov decision process (MDP).
In the classical setting, the decision maker is looking for a strategy (or ``policy'') that is optimal in terms of a \emph{risk-neutral} objective function.
Typically, this objective is the \emph{expected value} of some random cumulative return, accounting for both immediate and future rewards.
The dynamic programming (DP) approach comes with practical
algorithms to evaluate and optimize such objective functions, given the knowledge of the environment's dynamics (see
\citealp{puterman2014markov}).
DP is a popular framework for many applications ranging from computer programming to economics and inventory management:
we refer to \cite{bertsekas2000dynamic} for an overview.
Reinforcement learning (RL) aims at solving the same problem as DP when the dynamical model is unknown: the learner
only observes trajectories sampled from the MDP (see \citealp{sutton2018reinforcement}).

An inherent feature of standard dynamic programming procedures and most RL algorithms is their risk-neutrality,
meaning that they do not differentiate between strategies with the same expected return but different levels of risk
(for instance, different variances).
As a vanilla example, getting $0$ with probability (w.p.) $1$ is \emph{safer} than getting $+1$ w.p. $1/2$ and $-1$ w.p. $1/2$, though both scenarios are equivalent in expectation.
For this reason, standard DP and RL methods need to be adjusted to take risk into account, for instance by considering
alternative objective functions such as mean minus variance in
\cite{mannor2011mean}, \emph{conditional value-at-risk} (``CVaR'' in short) in \cite{osogami2012robustness} and \cite{chow2015risk}, or Chernoff functionals in \cite{moldovan2012risk}.

A more general approach, first proposed by \cite{morimura2010nonparametric} and \cite{morimura2012parametric}, is to
handle the whole distribution of the long-term return in a dynamic programming framework, not just its expectation or
some risk measure. This approach is fittingly called \emph{distributional} and is the main focus of our work.
Recently, \cite{bellemare2017distributional} introduced the distributional reinforcement learning (DRL) framework
and proposed the \textsc{C51} algorithm, achieving state-of-the-art performance in playing video games on the Atari 2600 benchmark \citep{bellemare2013arcade}.
\cite{rowland2018analysis} analyzed \textsc{C51} with the Cram\'er distance
and \cite{bellemare2019distributional} described another DRL algorithm that approximates distributions using the same
metric.
Many other DRL algorithms were proposed such as \textsc{QR-DQN} \citep{dabney2018distributional} and \textsc{IQN} \citep{dabney2018implicit} both based on quantile regression or \textsc{ER-DQN} in \cite{rowland2019distributional} based on expectile estimation.
Most of these DRL approaches rely on summarizing distributions by $N\ge 1$ \emph{atoms} $Q_1(x,a),\dots,Q_N(x,a)$
instead of the single state-action value function $Q(x,a)$ in classic RL or DP.
Although there are empirical evidence of the regularizing effect of learning several atoms in a function approximation setup \citep{lyle2019comparative}, finding an intuitive explanation for these quantities is still an open problem to the best of our knowledge.
Hence, it seems natural to ask the following question: ``Is there any meaningful interpretation of these atoms?''.
The answer provided by this paper is ``Yes, a \emph{robust MDP} interpretation!''.

The robust MDP framework is a seemingly unrelated way of dealing with uncertainties in sequential decision making,
with the main idea being the optimization of a worst-case objective function subject to an uncertainty set over the
true environment parameters (\citealp{iyengar2005robust}, \citealp{nilim2005robust}). In this work, we show that the
usual notion of robustness in MDPs can be directly derived from the \emph{distributional Bellman operator}
combined with a carefully chosen projection to a family of distributions involving only two atoms: this is our main
contribution. Additionally, we show that once all optimal policies have been identified and isolated from suboptimal
ones (by solving classical DP), our methodology allows a further discrimination among the space of optimal policies, by
distinguishing safe from risky optimal actions.

The rest of the paper is organized as follows. After providing the necessary technical
background in Section~\ref{sec:background}, we introduce our framework for robust distributional dynamic programming in
Section~\ref{sec:avar} and give an interpretation of the resulting value functions from the perspective of risk-measure
theory in Section~\ref{sec:interpretations}. Finally, in Section~\ref{sec:safe_risky}, we propose dynamic programming
methods for tiebreaking in the space of optimal policies to favor safe or risky policies, and provide some numerical
illustration to our results in Section~\ref{sec:experiments}. The paper is concluded with Section~\ref{sec:conclusion}.

\paragraph{Notations.}
Throughout the paper, we denote by $\mathbf{1}$ the all-ones vector (the dimensionality will always be clear from the context).
We let $\mathcal{P}_b(\mathbb{R})$ be the set of probability measures on $\mathbb{R}$ with bounded support, and $\mathcal{P}(\pazocal{E})$ the set of probability mass functions on any countable set $\pazocal{E}$, whose cardinality is denoted by $|\pazocal{E}|$.
The support of any discrete distribution $q\in \mathcal{P}(\pazocal{E})$ is: $ \text{Support}(q) = \{ y \in \pazocal{E} : q(y) > 0 \} $.
The cumulative distribution function (CDF) of a real-valued random variable $Z$ is the mapping $F(z)=\mathbb{P}(Z\le z)$
($\forall z\in\mathbb{R}$), and we denote its generalized inverse distribution function (a.k.a. quantile function) by
$F^{-1}: \tau\in (0, 1)\mapsto \inf\{z\in\mathbb{R}, F(z)\ge \tau\}$\footnote{We will often express the expectation of $Z$ with $F^{-1}$: Lemma \ref{lem:quantile_expectation} (in Appendix A) recalls the classic formula $\mathbb{E}[Z]=\int_{\tau=0}^1 F^{-1}(\tau)d\tau $.}.
For any probability measure $\nu\in\mathcal{P}_b(\mathbb{R})$
and measurable function $f:\mathbb{R}\rightarrow \mathbb{R}$, the \emph{pushforward measure} $\nu\circ f^{-1}$ is defined for any Borel set $A\subseteq \mathbb{R}$ by $\nu~\circ~f^{-1}(A) = \nu(\{ z \in \mathbb{R} : f(z) \in A \})$.
In this work, we will only encounter the affine case $f_{r_0,\gamma}(z)= r_0 + \gamma z$ (with $r_0\in\mathbb{R}, \gamma\in [0,1)$) for which
$\nu\circ f_{r_0,\gamma}^{-1}\in\mathcal{P}_b(\mathbb{R})$ and
$\nu\circ f_{r_0,\gamma}^{-1}(A) = \nu(\{ \frac{z-r_0}{\gamma} : z \in A \})$ if $\gamma\neq 0$,
or $\nu\circ f_{r_0,\gamma}^{-1} = \delta_{r_0}$ is the Dirac measure at $r_0$ if $\gamma=0$.
Lastly, for two probability measures $\nu,\nu'$ in $\mathcal{P}_b(\mathbb{R})$, $\nu\ll\nu'$ means that $\nu$ is absolutely continuous with respect to $\nu'$ (i.e. $\nu'(A)=0 \Rightarrow \nu(A)=0$)
and $\frac{d\nu}{d\nu'}$ is the Radon-Nikodym derivative.

\section{Background}\label{sec:background}
This section presents basic technical background on Markov decision processes, robust MDPs, and distributional dynamic
programming with $2$-Wasserstein projections.

\subsection{Markov decision process}
In this article, we study one of the most fundamental models for sequential decision-making problems: discounted Markov
decision processes (MDPs) with finite state and action spaces. Here we only describe the most essential elements of
this framework and refer to the classic textbook of \citet{puterman2014markov} for details.
A Markov decision process is described by the tuple $(\pazocal{X}, \pazocal{A}, P, r, \gamma)$ with
finite state space $\pazocal{X}$,
finite action space $\pazocal{A}$,
transition kernel $P: \pazocal{X}\times\pazocal{A} \rightarrow \mathcal{P}(\pazocal{X})$,
reward function $r: \pazocal{X}\times\pazocal{A}\times\pazocal{X} \rightarrow \mathbb{R}$
and discount factor $0 \le \gamma < 1$.
An MDP describes a sequential process where in each round of interaction, the decision-making agent chooses an action
$a\in\pazocal{A}$ while the environment occupies some state $x\in\pazocal{X}$, then the next state $X_1$ is sampled from
the distribution $P(\cdot | x, a)\in\mathcal{P}(\pazocal{X})$ and the agent gets the reward $r(x, a, X_1)$.
A \emph{stationary Markovian policy} $\pi: \pazocal{X}\rightarrow\mathcal{P}(\pazocal{A})$ maps any state $x$ to a distribution over the actions $\pi(\cdot|x)\in\mathcal{P}(\pazocal{A})$.
We denote by $\Pi$ the set of stationary Markovian policies.
The two major classes of problems in an MDP are the following.

\paragraph{The policy evaluation task:} the goal is to assess the quality of a policy $\pi$ in terms of expected return,
 through its state-action value function $Q^\pi$ defined for all $(x,a) \in \pazocal{X}\times\pazocal{A}$ as follows,
 \begin{equation*}
Q^\pi(x,a) = \mathbb{E}\left[ \sum_{t=0}^\infty \gamma^t r(X_t,A_t,X_{t+1}) \, \bigg| \, X_0=x, A_0=a \right] \,,
\end{equation*}
 where $X_{t+1}\sim P(\cdot|X_t,A_t)$ and $A_{t+1}\sim \pi(\cdot|X_{t+1})$.
 The Bellman equation verified by $Q^\pi$ is
 \begin{equation*}
   Q^\pi(x,a) = \sum_{(x',a')\in\pazocal{X}\times\pazocal{A}} P(x'|x,a)\pi(a'|x') \left( r(x,a,x') + \gamma Q^\pi(x',a') \right) .
 \end{equation*}
 The corresponding value function is $V^\pi(x)=\sum_a \pi(a|x) Q^\pi(x,a)$.

\paragraph{The control task:} find an optimal policy $\pi^*$ that simultaneously maximizes the values across all
states,
 \begin{equation*}
   V^{\pi^*}(x) = \sup_{\pi\in\Pi} V^\pi(x) =: V^*(x)
   \quad \text{ and } \quad
   Q^{\pi^*}(x,a) = \sup_{\pi\in\Pi} Q^\pi(x,a) =: Q^*(x,a) \,,
 \end{equation*}
 where $Q^*$ satisfies the Bellman optimality equation
 \begin{equation*}
   Q^*(x,a) = \sum_{x'} P(x'|x,a) \left( r(x,a,x') + \gamma \max_{a'} Q^*(x',a') \right) .
 \end{equation*}
 Moreover $V^*(x)=\max_a Q^*(x,a)$, where the maximizing actions
 include the support of any optimal policy, in any state:
 \begin{equation*}
  \pi^* \text{ is optimal } \quad \text{ if and only if } \quad \forall x,\, \text{Support}(\pi^*(\cdot | x)) \subseteq \pazocal{A}^*(x) := \argmax_{a} Q^*(x,a) \,.
 \end{equation*}

Equivalently, $Q^\pi$ (resp.~$Q^*$) can be seen as the unique fixed point of the \emph{Bellman operator} (resp.~\emph{Bellman optimality operator}),
which is a $\gamma$-contraction\footnote{A function mapping a metric space to itself is called a $\gamma$-contraction if
it is Lipschitz continuous with Lipschitz constant $\gamma<1$.} in supremum norm\footnote{The supremum norm of any
function $h=(h_1,\dots,h_k):\pazocal{X}\times\pazocal{A}\rightarrow \mathbb{R}^k$ is $||h||_\infty =
\sup_{(x,a,i)\in\pazocal{X}\times\pazocal{A}\times\{1,\dots,k\}} |h_i(x,a)|$.} \citep{bertsekas1996neuro}.
The dynamic programming (DP) approach consists in recursively applying these contractive operators
until convergence to their fixed points (see \citealp{bellman57dynamic}).
In Section~\ref{subsec:DBO}, we recall how the same idea can be extended to entire probability distributions, not just
expected values.

\begin{figure}
\centering
\includegraphics[width=0.8\linewidth]{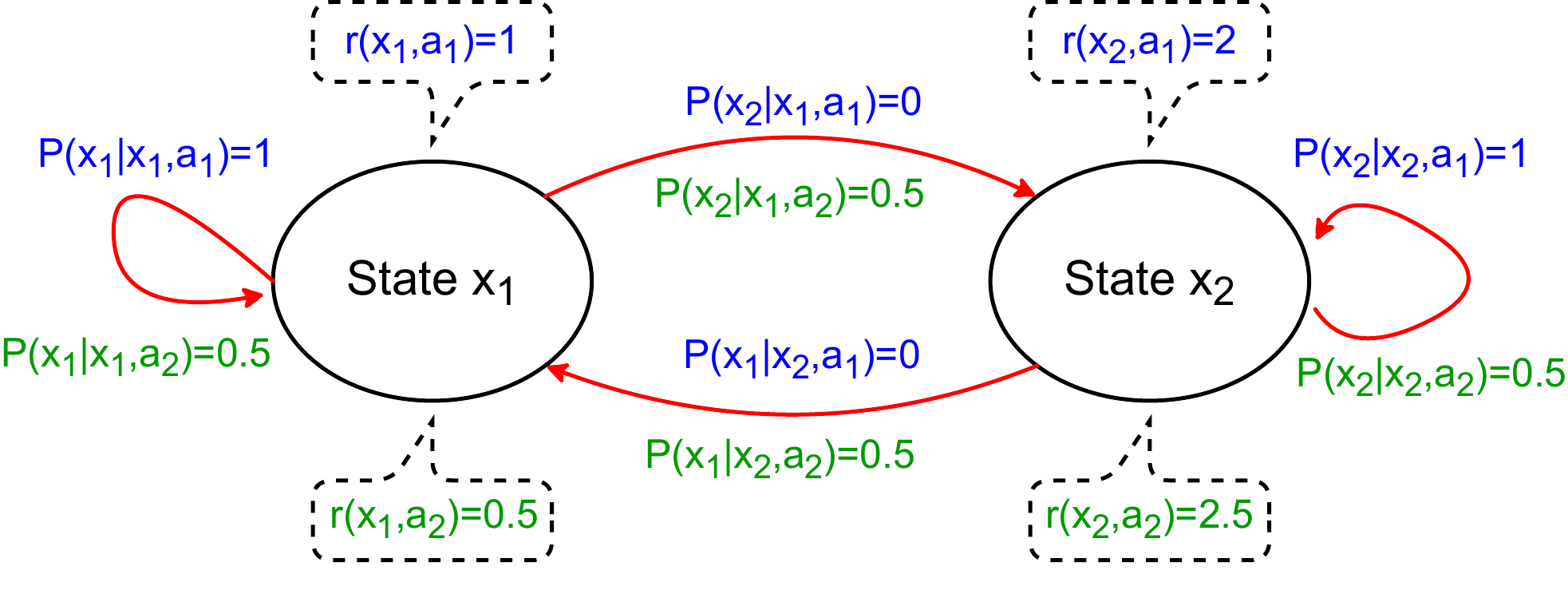}
\caption{Markov decision process with two states $\pazocal{X}=\{x_1,x_2\}$ and two actions $\pazocal{A}=\{a_1,a_2\}$. In this example, the reward function does not depend on the next state: $r(x,a,\cdot)\equiv r(x,a)$.}
\label{fig:ex_mdp}
\end{figure}

\subsection{Robust MDPs}
\label{subsec:robust_mdps}

Sometimes, there may be uncertainty in the transition probabilities,
for instance when they are estimated from noisy observations, or when the state transitions can be influenced by
an external adversary.
The robust MDP framework models this situation with an \emph{uncertainty set} $\Upsilon$ consisting of a family of
candidate transition kernels $\boldsymbol{P}$.
The worst-case value function of a policy $\pi$ with respect to this set is defined for each state $s$ as
\begin{equation}
  \label{eq:worst}
  V_{\text{worst}}^\pi(s) = \inf_{\boldsymbol{P}\in\Upsilon} V_{\boldsymbol{P}}^\pi(s) \,,
\end{equation}
where $V_{\boldsymbol{P}}^\pi$ denotes the value function in the MDP with kernel $\boldsymbol{P}$.
Then, given an initial state $s_0$, the goal is to find a robust optimal policy $\pi^*$ satisfying
\begin{equation*}
  V_{\text{worst}}^{\pi^*}(s_0)\ge V_{\text{worst}}^\pi(s_0) \quad \text{for all } \pi \, .
\end{equation*}
This setting has been extensively studied in the literature under various assumptions for
the uncertainty set.
In \cite{iyengar2005robust} and \cite{nilim2005robust}, the uncertainty set $\Upsilon$
is a Cartesian product over state-action pairs:
\begin{equation*}
  \Upsilon = \bigtimes_{s,a} \Upsilon_{s,a} = \{ \boldsymbol{P}=(\boldsymbol{P}(\cdot|s,a))_{s,a} : \boldsymbol{P}(\cdot|s,a) \in \Upsilon_{s,a} \} .
\end{equation*}
In other words, each component $\boldsymbol{P}(\cdot|s,a)$ can be chosen independently
among the set $\Upsilon_{s,a}$ in the infimum in Eq. (\ref{eq:worst}).
This is the $(s,a)$-rectangularity assumption, under which there exists a robust optimal policy that is stationary, Markovian and deterministic, that can be computed by robust dynamic programming.
Similarly, the $s$-rectangularity assumption is considered in \cite{wiesemann2013robust}:
\begin{equation*}
  \Upsilon = \bigtimes_{s} \Upsilon_{s} = \{ \boldsymbol{P}=(\boldsymbol{P}(\cdot|s,\cdot))_{s} : \boldsymbol{P}(\cdot|s,\cdot) \in \Upsilon_{s} \} .
\end{equation*}
Under this weaker assumption, there is a stationary Markovian robust optimal policy, but (unfortunately) it may not be deterministic.
More general assumptions have also been studied.
Factor matrix uncertainty sets along with the so-called ``$r$-rectangularity'' structure
are investigated in \cite{goyal2018robust}. This alternative hypothesis is proved to generalize
the $(s,a)$-rectangulariy condition, and to produce as well a deterministic robust optimal policy.
In \cite{mannor2016robust}, the authors consider a generalization
of $s$-rectangularity, called $k$-rectangularity.
In section \ref{sec:interpretations}, we show that our distributional approach reformulates as a robust MDP outside of any of the aforementioned rectangulariy assumptions.
Further in section \ref{sec:safe_risky}, our setting leads to a deterministic robust optimal policy.
But first, we need to recall the definition of the distributional Bellman operator, which is the main
tool to handle distributions in MDPs.

\subsection{The distributional Bellman operator}
\label{subsec:DBO}

The \emph{distributional Bellman operator} (DBO) was introduced in \cite{morimura2010nonparametric} and \cite{morimura2012parametric} for CDFs, in \cite{bellemare2017distributional} with random variables;
here we recall its formulation based on \emph{pushforward measures} from \cite{rowland2018analysis}.
On a high level, the DBO takes as input a \emph{distribution function}
$\mu \in \mathcal{P}_b(\mathbb{R})^{\pazocal{X}\times\pazocal{A}}$
that models the collection of return distributions indexed by
state-action pairs, and returns another distribution function $\pazocal{T}^\pi \mu$ corresponding to the distribution
of returns after being pushed through the transition dynamics. The more formal definition is the following:
\begin{definition}{\textsc{(Distributional Bellman operator)}.}
  \label{def:DBO}
  Let $\pi\in\Pi$. The distributional Bellman operator $\pazocal{T}^\pi : \mathcal{P}_b(\mathbb{R})^{\pazocal{X}\times\pazocal{A}} \rightarrow \mathcal{P}_b(\mathbb{R})^{\pazocal{X}\times\pazocal{A}}$ is defined for any distribution function $\mu=~(\mu^{(x,a)})_{x,a}$ by
  \begin{equation*}
    (\pazocal{T}^\pi \mu)^{(x,a)} = \sum_{(x',a')\in\pazocal{X}\times\pazocal{A}} P(x'|x,a)\pi(a'|x') \cdot \mu^{(x',a')}\circ f_{r(x,a,x'),\gamma}^{-1} \quad , \quad
    \text{ for all } (x,a)\in\pazocal{X}\times\pazocal{A} \,.
  \end{equation*}
\end{definition}

In particular, the DBO is mixture-linear: for all $\mu_1,\mu_2$ and $0\le \lambda\le 1$,
\begin{equation*}
  \pazocal{T}^\pi (\lambda \mu_1+(1-\lambda)\mu_2) = \lambda \pazocal{T}^\pi \mu_1 + (1-\lambda) \pazocal{T}^\pi \mu_2 .
\end{equation*}
It was proved in \cite{bellemare2017distributional} that $\pazocal{T}^\pi$ is a $\gamma$-contraction
in the maximal $p$-Wasserstein metric\footnote{We recall that the $p$-Wasserstein distance ($p\ge 1$) between two
probability distributions $\nu_1,\nu_2$ on $\mathbb{R}$ with CDFs $F_1,F_2$ is defined as
$ W_p(\nu_1, \nu_2) = \left( \int_{\tau=0}^1 \left| F_1^{-1}(\tau) - F_2^{-1}(\tau) \right|^p d\tau\right)^\frac{1}{p}
$, and for $p=\infty$ as $ W_\infty(\nu_1, \nu_2) = \sup_{\tau\in(0, 1)}
|F_1^{-1}(\tau) - F_2^{-1}(\tau)| $.}
\begin{equation*}
\widetilde{W}_p(\mu_1, \mu_2) = \max_{(x,a)\in\pazocal{X}\times\pazocal{A}} W_p(\mu_1^{(x, a)}, \mu_2^{(x, a)})
\end{equation*}
at any order $p\in [1, +\infty]$.
Akin to the ``non-distributional'' case, we know from Banach's fixed point theorem that
iterating the distributional Bellman operation, namely $\pazocal{T}^\pi \circ \dots \circ \pazocal{T}^\pi \mu$ (starting
from an arbitrary initial $\mu$), defines a sequence that converges exponentially fast to the unique fixed point
$\mu_\pi = \pazocal{T}^\pi \mu_\pi$. This equality is called the \emph{distributional Bellman
equation}.
The fixed point $\mu_\pi=(\mu_\pi^{(x,a)})_{x,a}$ is the collection of the probability laws of the returns:
\begin{equation}
  \label{eq:law_return}
\text{for all } (x,a), \quad  \mu_\pi^{(x,a)} = \text{Law}\left( \sum_{t=0}^\infty \gamma^t r(X_t,A_t,X_{t+1}) \, \bigg|\, X_0=x,A_0=a \,;\,\pi \right) \,.
\end{equation}
Nevertheless, it may be hard in practice to compute $\pazocal{T}^\pi \mu$,
as it requires to deal with general distributions.
In the example below, we focus on a basic family of probability distributions: the \emph{atomic distributions}, over which the DBO is stable.
\begin{example}{\textsc{(``Atomic fission'')}.}
  \label{ex:atomic}
  Let $\mu=(\mu^{(x,a)})_{x,a}$ be an atomic distribution function with
  \begin{equation*}
    \mu^{(x,a)} = \sum_{i=1}^N \alpha_i(x,a) \delta_{Q_i(x,a)} \,,
  \end{equation*}
  where $N\ge 1$, $\alpha_1(x,a),\dots,\alpha_N(x,a) \ge 0$, $\alpha_1(x,a)+\dots+\alpha_N(x,a)=1$.
  Then for any $(x,a)$,
  \begin{equation*}
    (\pazocal{T}^\pi \mu)^{(x,a)} = \sum_{(x',a')\in\pazocal{X}\times\pazocal{A}} P(x'|x,a)\pi(a'|x') \sum_{i=1}^N \alpha_i(x',a') \delta_{r(x,a,x') + \gamma Q_i(x',a')} \,,
  \end{equation*}
  which is still an atomic distribution, but with up to $|\pazocal{X}||\pazocal{A}|$ times more particles.
\end{example}

\begin{figure}

\centering
\begin{subfigure}{.49\textwidth}
  \centering
  \includegraphics[width=1.0\linewidth]{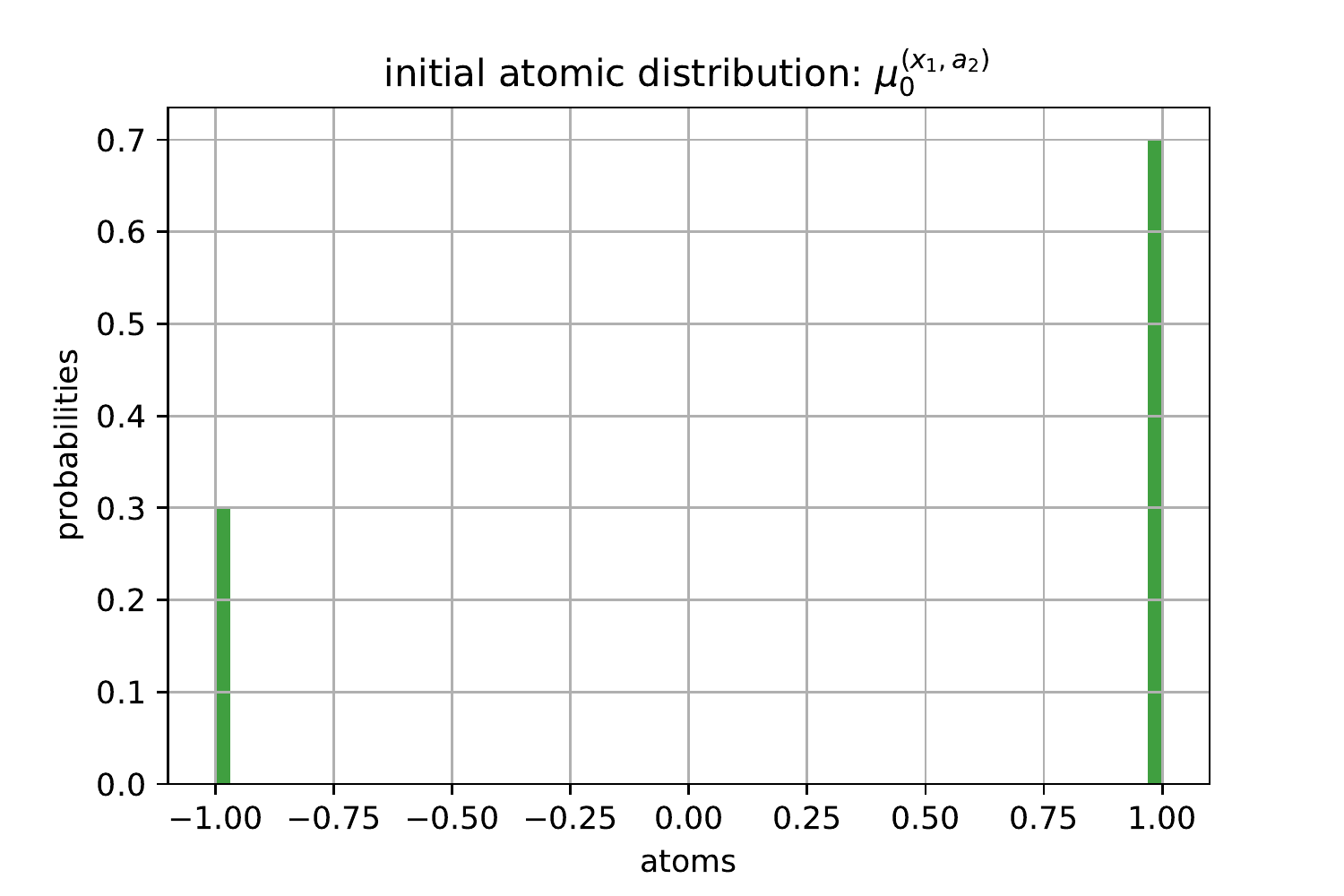}
\end{subfigure}
\begin{subfigure}{.49\textwidth}
  \centering
  \includegraphics[width=1.0\linewidth]{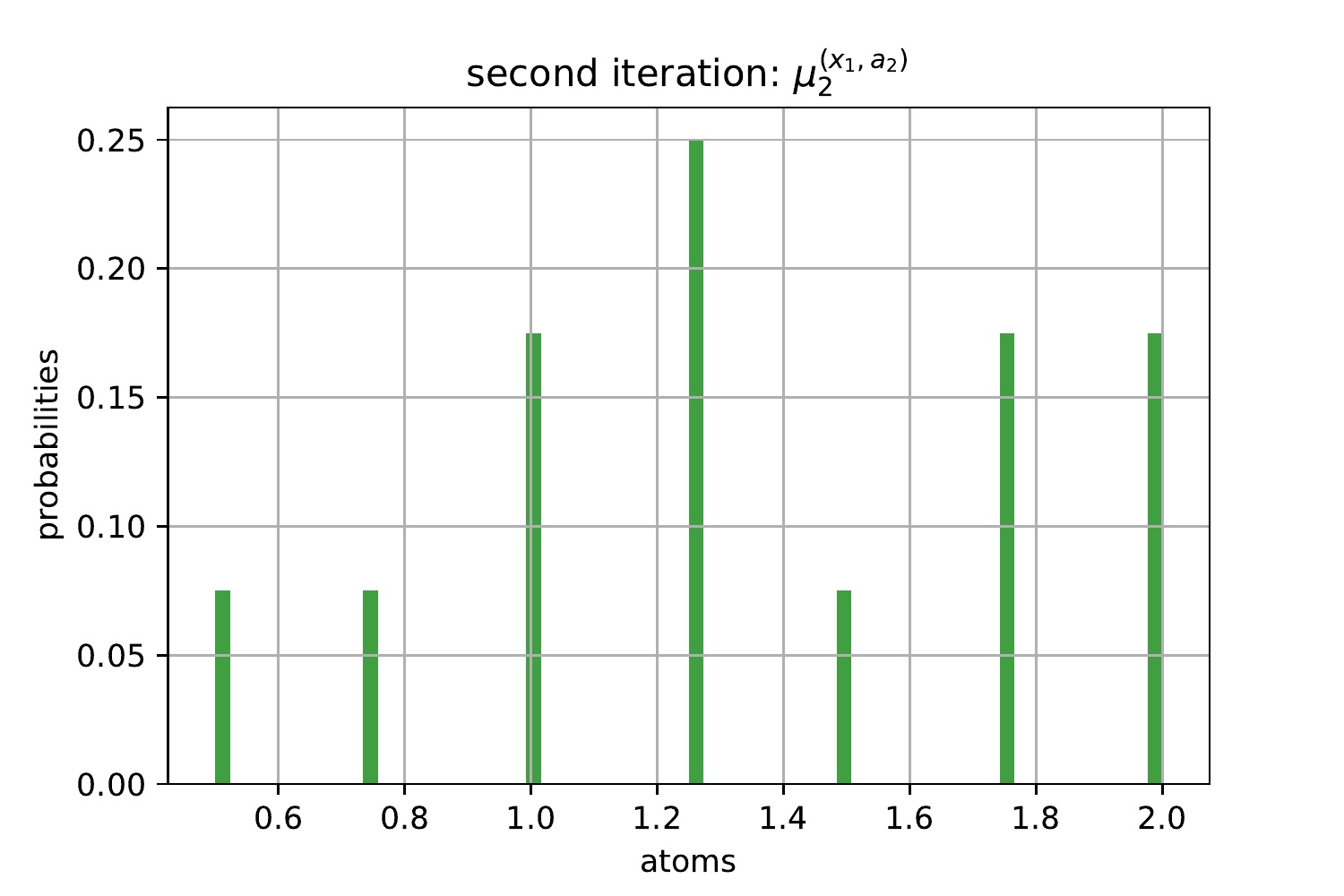}
\end{subfigure}
\begin{subfigure}{.49\textwidth}
  \centering
  \includegraphics[width=1.0\linewidth]{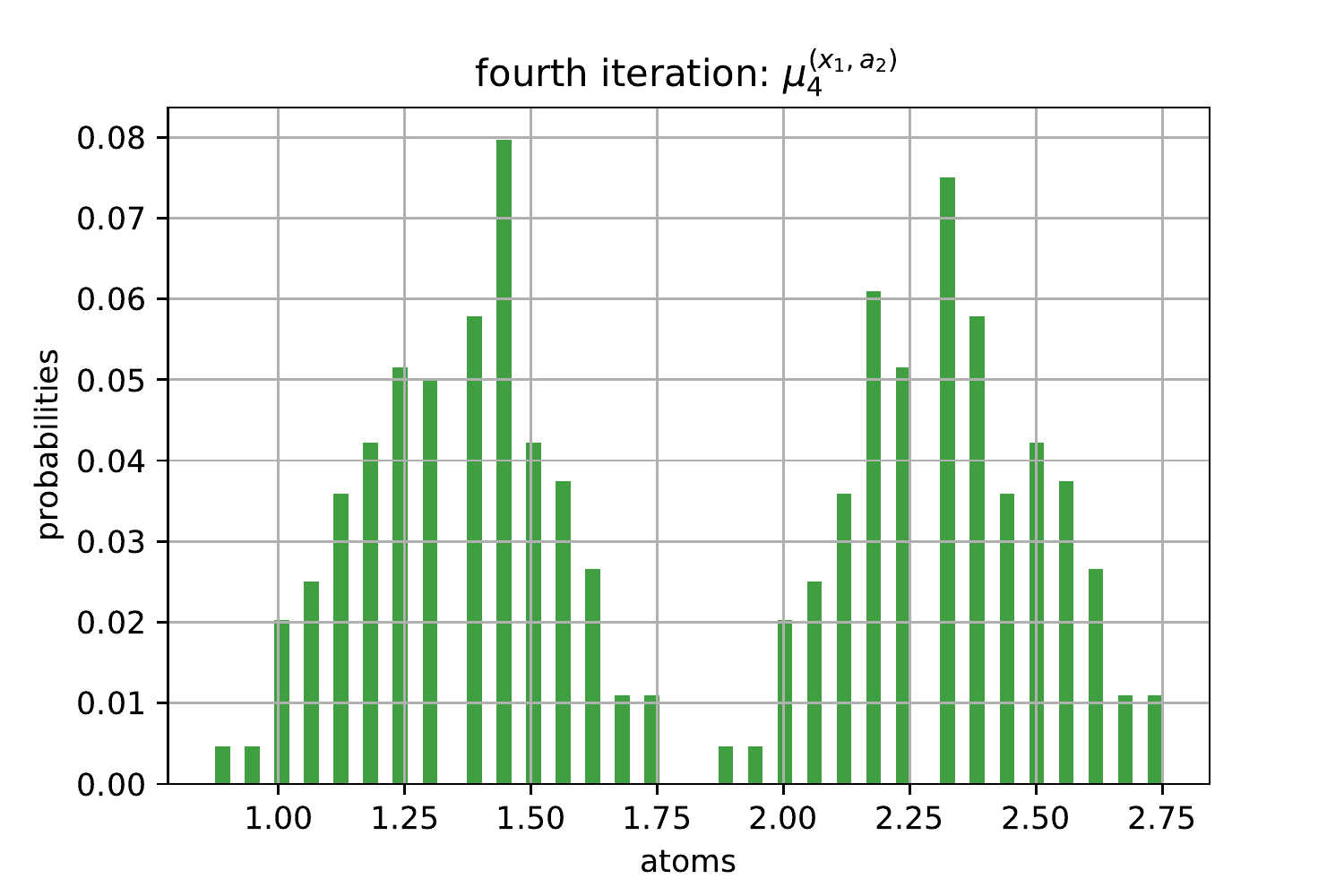}
\end{subfigure}

\caption{Illustration of Example \ref{ex:atomic} in the MDP in Figure \ref{fig:ex_mdp}, with discount factor $\gamma=\frac{1}{2}$,
for the stochastic policy $\pi(a|x)=\frac{1}{2}$ for any $x\in\{x_1,x_2\},a\in\{a_1,a_2\}$.
For $k\ge 0$, the distribution function $\mu_k = (\pazocal{T}^\pi)^k \mu_0$ is obtained by applying $k$ times the DBO
to the initial atomic distributions $\mu_0^{(x,a)}=0.3\delta_{-1} + 0.7\delta_{1}$ (for all $x,a$).}
\label{fig:atomic}
\end{figure}

Motivated by Example \ref{ex:atomic}, several distributional RL algorithms are based on atomic distributions: they apply
$\pazocal{T}^\pi$---which multiplies the number of particles by a factor up to $|\pazocal{X}||\pazocal{A}|$---followed by a
projection to bring the number of particles back to a fixed budget $N$.

\noindent \textbf{Projected operators.}
In \cite{dabney2018distributional}, every distribution $(\pazocal{T}^\pi \mu)^{(x,a)}$ with CDF $F_{x,a}$
is projected to a discrete distribution with uniform probabilities over $N\ge 1$ evenly spread quantiles:
\begin{equation*}
  \frac{1}{N} \sum_{i=1}^N \delta_{Q_i(x,a)} \quad \text{ with } \quad Q_i(x,a) = F_{x,a}^{-1}\left( \frac{2i-1}{2N} \right) \,.
\end{equation*}
This specific choice comes from minimizing the \emph{$\mathit{1}$-Wasserstein distance} between $(\pazocal{T}^\pi \mu)^{(x,a)}$ and such average of Dirac measures.
Notice that in the monoatomic case $N=1$, this approach summarizes an entire distribution by a single scalar: the median
$F_{x,a}^{-1}(1/2)$. Interestingly, this suggests that these projected operators are not an appropriate generalization
of the classical Bellman operators that involve the \emph{expectation} of the return instead of the median.

This issue can be addressed by using the \emph{$\mathit{2}$-Wasserstein distance} for projection:
\begin{equation*}
  W_2((\pazocal{T}^\pi \mu)^{(x,a)} , \delta_{Q(x,a)} )^2 = \int_{\tau=0}^1 (F_{x,a}^{-1}(\tau)-Q(x,a))^2 d\tau \,.
\end{equation*}
Indeed, this $W_2$-error is a quadratic function in $Q(x,a)$, whose minimum is attained at the mean value
$Q(x,a)=\int_{\tau=0}^1 F_{x,a}^{-1}(\tau)d\tau = \mathbb{E}_{Z\sim (\pazocal{T}^\pi \mu)^{(x,a)}}[Z]$ (by Lemma
\ref{lem:quantile_expectation}). Thus, combining the distributional Bellman operators with a 2-Wasserstein projection
correctly recovers the classical DP operators in the special monoatomic case $N=1$.
More formally, it is easy to check that averaging after applying the DBO to a collection $\delta_Q=(\delta_{Q(x,a)})_{x,a}$ of Dirac measures,
is nothing but the usual policy evaluation update:
\begin{equation*}
\mathbb{E}_{Z\sim (\pazocal{T}^\pi \delta_Q)^{(x,a)}}\left[ Z \right] = \sum_{x',a'} P(x'|x,a)\pi(a'|x') \left( r(x,a,x') + \gamma Q(x',a') \right) .
\end{equation*}
This simple observation motivated $W_2$-projections in \cite{achab2020ranking} (see chapter VII therein),
who derived multiatomic variants of the Temporal-Difference and Q-learning algorithms.
As shall be seen in the next section, this choice leads to a natural extension of non-distributional DP with closed-form updates
even for more than one atom.

\section{Risk-sensitive distributional dynamic programming with 2-Wasserstein projections}
\label{sec:avar}
We now turn to describing our main contribution: a framework for risk-sensitive dynamic programming using
$W_2$-projected distributional Bellman operators, arising as a special case of the distributional DP framework
described in the previous section using projections onto the set of diatomic
distributions with fixed non-uniform weights. As we will show, projection to this set corresponds to calculating the
well-studied coherent risk measure of \emph{average value-at-risk} or ``AVaR''\footnote{AVaR and CVaR are two different
names for the same quantity \citep{chun2012conditional}.} (see \citealp{rockafellar2000optimization},
\citealp{rockafellar2002conditional} and \citealp{acerbi2002coherence}). We start with the formal definitions below.

\begin{definition}{\textsc{(Average value-at-risk)}.}
  \label{def:avar}
  Let $0<\alpha<1$ and $\nu\in\mathcal{P}_b(\mathbb{R})$ with CDF $F$. The left and right AVaRs of $\nu$ at respective levels $\alpha$ and $1-\alpha$ are defined as:
  \begin{equation*}
    \text{AVaR}^{\text{left}}_{\alpha}(\nu) = \frac{1}{\alpha} \int_{\tau=0}^\alpha F^{-1}(\tau) d\tau
    \quad \text{ and } \quad \text{AVaR}^{\text{right}}_{1-\alpha}(\nu) = \frac{1}{1-\alpha} \int_{\tau=\alpha}^1 F^{-1}(\tau) d\tau \,.
  \end{equation*}
\end{definition}

\begin{definition}{\textsc{(Diatomic distribution function)}.}
  \label{def:diatomic_distrib}
  Given $\alpha\in (0, 1)$ and some bidimensional function
  $\mathcal{Q}=(Q_1,Q_2):\pazocal{X}\times\pazocal{A}\rightarrow \mathbb{R}^2$,
  we call ``diatomic distribution function'' and denote $D_{\alpha,\mathcal{Q}}=~(D_{\alpha,\mathcal{Q}}^{(x,a)})_{(x,a)\in\pazocal{X}\times\pazocal{A}}$
  the following collection of diatomic distributions:
  \begin{equation*}
  D_{\alpha,\mathcal{Q}}^{(x,a)}=\alpha \delta_{Q_1(x,a)} + (1-\alpha) \delta_{Q_2(x,a)} \quad \, , \quad \text{ for all } (x,a)\,.
  \end{equation*}
\end{definition}

In other words, we associate a mixture $D_{\alpha,\mathcal{Q}}^{(x,a)}$ of two Dirac masses to each state-action pair $(x,a)$.
As a comparison, this doubles the space complexity of the usual DP framework that uses a single action-value function
$Q(x,a)$ in place of $\mathcal{Q}(x,a)=(Q_1(x,a),Q_2(x,a))$.
Although we focus on distributions made of two atoms, most of the techniques used in this paper easily extend to any number of atoms.


In order to respect our diatomic constraint, we apply successively
the distributional Bellman operator and a projection onto the space of diatomic distribution functions.
More precisely, every time the operator $\pazocal{T}^\pi$ is applied to some $D_{\alpha,\mathcal{Q}}$,
we use the $2$-Wasserstein-projection to approximate $\pazocal{T}^\pi D_{\alpha,\mathcal{Q}}$ by another
collection of distributions in the same family $D_{\alpha,\mathcal{Q}'}$. The following lemmas give a tractable method
for evaluating the resulting operator.

\begin{lemma}{\textsc{(From $W_2$-projection to AVaR)}.}
  \label{lem:W2proj}
  Let $0<\alpha<1$ and $\nu\in\mathcal{P}_b(\mathbb{R})$. Then, there exists a unique couple $(\theta^*_1,\theta^*_2)\in\mathbb{R}^2$ minimizing the $2$-Wasserstein approximation error
  $\min_{\theta_1 \le \theta_2} \, W_2(\nu, \, \alpha \delta_{\theta_1} + (1-\alpha) \delta_{\theta_2})$
  between $\nu$ and any diatomic proxy.
  In addition, this best diatomic approximation is given by the left and right AVaRs of $\nu$, at levels $\alpha$ and $1-\alpha$ respectively:
  \begin{equation*}
    \theta^*_1 = \text{AVaR}^{\text{left}}_{\alpha}(\nu)
    \quad \text{ and } \quad \theta^*_2 = \text{AVaR}^{\text{right}}_{1-\alpha}(\nu) \,.
  \end{equation*}
\end{lemma}

\begin{proof}
For any $\theta=(\theta_1,\theta_2)\in\mathbb{R}^2$ with $\theta_1\le \theta_2$,
first notice that the CDF $F_{\alpha,\theta}$
and the quantile function $F_{\alpha,\theta}^{-1}$ of the diatomic distribution $D_{\alpha,\theta}=\alpha \delta_{\theta_1} + (1-\alpha) \delta_{\theta_2}$ are given by:
\begin{multline*}
\forall y\in\mathbb{R} , \quad F_{\alpha,\theta}(y) = \alpha \mathbb{I}\{\theta_1 \le y\} + (1-\alpha) \mathbb{I}\{\theta_2 \le y\}\\
\quad \text{ and } \quad \forall \tau\in (0,1) , \quad F_{\alpha,\theta}^{-1}(\tau) = \mathbb{I}\left\{0 < \tau \le \alpha \right\} \theta_1 + \mathbb{I}\left\{\alpha < \tau \le 1 \right\} \theta_2 \, .
\end{multline*}

By definition of the $2$-Wasserstein distance, we have:
\begin{equation*}
  W_2(\nu, \, D_{\alpha,\theta})^2
  = \int_{\tau=0}^1 (F^{-1}(\tau) - F_{\alpha,\theta}^{-1}(\tau))^2 d\tau
  = \int_{\tau=0}^\alpha (F^{-1}(\tau) - \theta_1)^2 d\tau + \int_{\tau=\alpha}^1 (F^{-1}(\tau) - \theta_2)^2 d\tau ,
\end{equation*}
where the first term (which only depends on $\theta_1$) rewrites:
\begin{equation*}
  \int_{\tau=0}^\alpha (F^{-1}(\tau) - \theta_1)^2 d\tau
  = \int_{\tau=0}^\alpha F^{-1}(\tau)^2 d\tau - 2 \theta_1 \int_{\tau=0}^\alpha F^{-1}(\tau) d\tau
  + \alpha \theta_1^2 ,
\end{equation*}
which is minimized for $\theta_1 = (1/\alpha) \int_{\tau=0}^\alpha F^{-1}(\tau) d\tau $.
Similarly, the second term is minimal if and only if $\theta_2 = (1-\alpha)^{-1} \int_{\tau=\alpha}^1 F^{-1}(\tau) d\tau$, which concludes the proof.
\end{proof}
The $W_2$-projection in Lemma \ref{lem:W2proj} appears as a natural and canonical choice:
it is simply an orthogonal projection in the $L^2$ space of quantile functions.
In the following, we will apply it \emph{entrywise}, i.e. by computing the left and right AVaRs for each of the $|\pazocal{X}||\pazocal{A}|$ distributions $\nu=(\pazocal{T}^\pi D_{\alpha,\mathcal{Q}})^{(x,a)}$ contained in $\pazocal{T}^\pi D_{\alpha,\mathcal{Q}}$.



\paragraph{Key observation: discrete AVaR.}
Although the AVaR may not be easy to compute for general distributions, luckily our approach only requires it for
discrete distributions. Indeed,
\begin{equation}
  \label{eq:DBO_discrete}
(\pazocal{T}^\pi D_{\alpha, \mathcal{Q}})^{(x,a)} = \sum_{(x',a')\in\pazocal{X}\times\pazocal{A}} P(x'|x,a)\pi(a'|x') \left( \alpha \delta_{r(x,a,x')+\gamma Q_1(x',a')} + (1-\alpha) \delta_{r(x,a,x')+\gamma Q_2(x',a')} \right)
\end{equation}
is simply a weighted average of $2|\pazocal{X}||\pazocal{A}|$ Dirac masses.
Plus, the AVaR of a discrete distribution comes with a closed-form expression provided below.

\begin{lemma}{\textsc{(AVaR of a discrete distribution)}.}
\label{lem:avar}
Let $0 < \alpha < 1$ and $\nu$ be a discrete distribution:
\begin{equation*}
\nu = \sum_{j=1}^M p_j \delta_{v_j} \,,
\end{equation*}
with $M\ge 1$, $p_1,\dots,p_M \ge 0$, $p_1+\dots+p_M=1$ and sorted values $v_1\le\dots\le v_M$.
\begin{enumerate}[(i)]
\item Closed-form expression: the left and right AVaRs of $\nu$ at respective levels $\alpha$ and $1-\alpha$ are
\begin{multline*}
  \text{AVaR}^{\text{left}}_{\alpha}(\nu) = \frac{1}{\alpha} \sum_{j=1}^M \max\left( 0\,,\,\min\left( p_j \,,\,\alpha - \sum_{j'\le j-1} p_{j'} \right) \right) \cdot v_j \\
\text{ and } \quad \text{AVaR}^{\text{right}}_{1-\alpha}(\nu) = \frac{1}{1-\alpha} \sum_{j=1}^M \max\left( 0\,,\,\min\left( p_j \,,\,\sum_{j'\le j} p_{j'} - \alpha \right) \right) \cdot v_j \quad ,
\end{multline*}
with the empty sum convention $\sum_{j'\le 0} p_{j'}=0$ .
\item Dual representation: denoting $v=(v_1,\dots, v_M)$ and $p=(p_1,\dots, p_M)$,
\begin{equation*}
  \text{AVaR}^{\text{left}}_{\alpha}(\nu) = \frac{1}{\alpha} \inf_{\lambda} \langle \lambda, v \rangle \quad
\text{ and } \quad \text{AVaR}^{\text{right}}_{1-\alpha}(\nu) = \frac{1}{1-\alpha} \sup_{\lambda} \langle p - \lambda, v \rangle \quad ,
\end{equation*}
where the infimum and supremum both range over weights $\lambda=(\lambda_1,\dots,\lambda_M)$ such that
\begin{equation*}
\begin{cases}
  & \forall i, \quad 0 \le \lambda_i \le p_i \\
  & \sum_{i=1}^M \lambda_i = \alpha \ .
\end{cases}
\end{equation*}
\end{enumerate}
\end{lemma}

\begin{proof}
\noindent (i) Closed-form expression. First denote $\overline{p}_j = \sum_{j'=1}^j p_j$ for each $j\in \{1,\dots,M\}$, and $\overline{p}_0=0$.
Then, the CDF $F$ and quantile function $F^{-1}$ of the discrete distribution $\nu$ are:
\begin{equation*}
\forall y\in\mathbb{R} , \quad F(y) = \sum_{j=1}^M p_j \mathbb{I}\{v_j \le y\} \\
\quad \text{ and } \quad \forall \tau\in (0,1) , \quad
F^{-1}(\tau) = \sum_{j=1}^M \mathbb{I}\left\{\overline{p}_{j-1} < \tau \le \overline{p}_j \right\} v_j \, ,
\end{equation*}
which are both step functions.
Hence, the right AVaR of $\nu$ at level $1-\alpha$ simply writes:
\begin{equation*}
\text{AVaR}^{\text{right}}_{1-\alpha}(\nu)
= \int_{\tau=\alpha}^1 F^{-1}(\tau) d\tau
= \sum_{j=1}^M \text{Length}\left( [\alpha,1] \cap [\overline{p}_{j-1}, \overline{p}_j] \right) v_j .
\end{equation*}
Observing that the length of the intersection of two intervals $[a,b]$ and $[c,d]$ is equal to
\begin{equation*}
  \text{Length}\left( [a, b] \cap [c, d] \right) = \max(0 \,,\, \min(b, d) - \max(a, c) )
\end{equation*}
concludes the proof.

\noindent (ii) Dual representation. If $\alpha\in [0, p_1]$, then $\inf_{\lambda} \langle \lambda, v \rangle = \alpha v_1$.
For $j\ge 2$, if $\alpha\in [\overline{p}_{j-1}, \overline{p}_j]$, then obviously
$\inf_{\lambda} \langle \lambda, v \rangle = (\alpha-\overline{p}_{j-1}) v_j + \sum_{i=1}^{j-1} p_i v_i$.
All these different cases coincide with the expression provided in (i) for the left AVaR.
We conclude the proof by observing that $\alpha \text{AVaR}^{\text{left}}_{\alpha}(\nu) + (1-\alpha) \text{AVaR}^{\text{right}}_{1-\alpha}(\nu) = \langle p, v \rangle$ and that both the infimum and the supremum in (ii) are attained at the same $\lambda$.
\end{proof}

Figure \ref{fig:discreteAVaR} depicts an application of Lemma \ref{lem:avar}: the left (resp.~right) AVaR is obtained
by computing the signed area delimited by the staircase curve of the quantile function on the segment $[0, \alpha]$ (resp.~$[\alpha,1]$).
The reason why we have the closed-form formula in Lemma \ref{lem:avar}-(i) is because this area is just the sum of the areas of rectangles.


\begin{figure}
\centering
\includegraphics{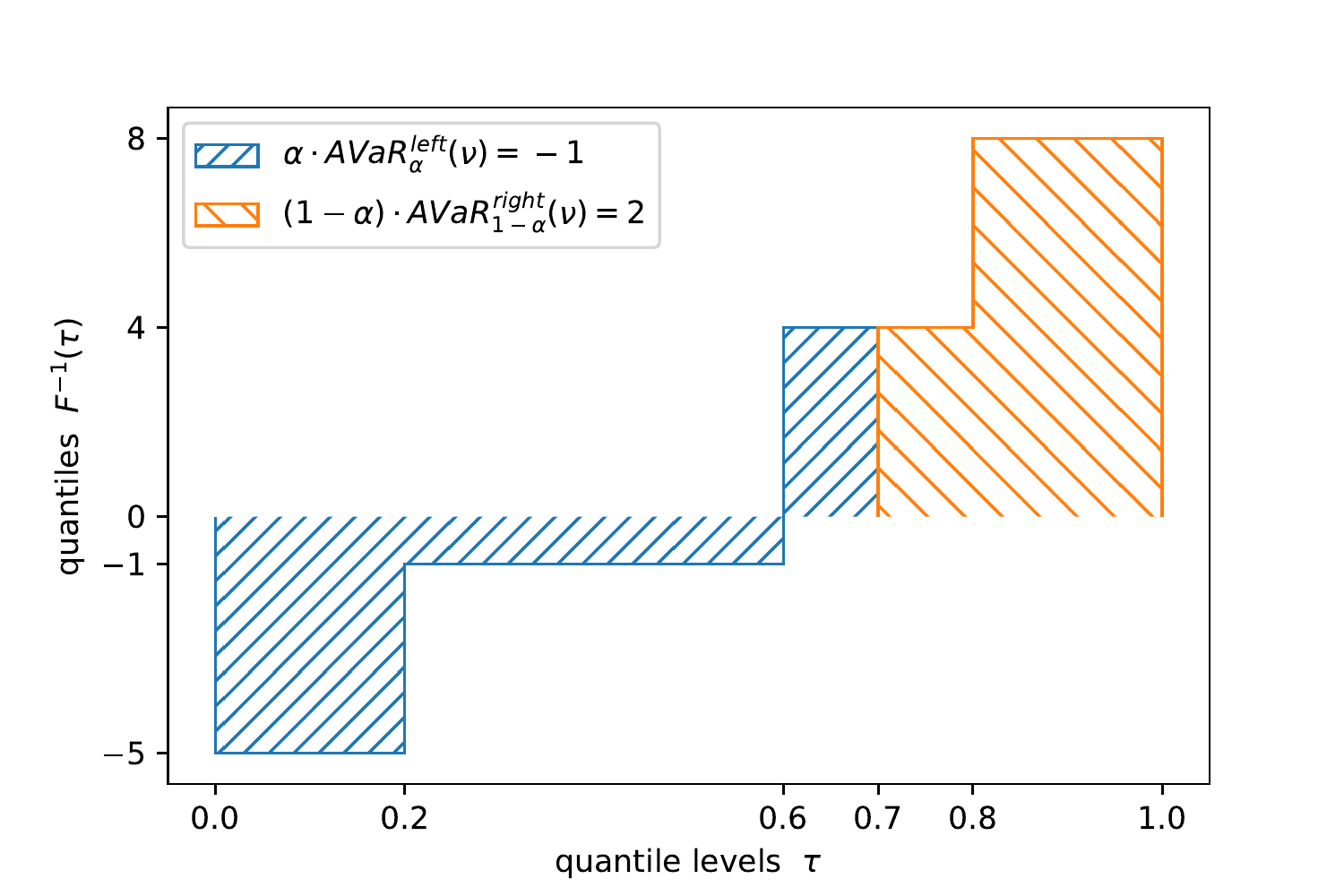}
\caption{Left and right average value-at-risk of a discrete distribution. Graphic illustration of Lemma \ref{lem:avar}
for $\nu=0.2 (\delta_{-5}+2\delta_{-1}+\delta_{4}+\delta_{8})$ (with piecewise constant quantile function $F^{-1}$) and
AVaR level $\alpha=0.7$. The total shaded (signed) area on the left is equal to $-1$; the rightmost one to $2$.}
\label{fig:discreteAVaR}
\end{figure}

\subsection{Diatomic policy evaluation}

Using Lemma \ref{lem:W2proj}, our distributional DP method can be formulated with operators that map functions
$(Q_1,Q_2)\in\mathbb{R}^{2|\pazocal{X}||\pazocal{A}|}$ to $(Q_1',Q_2')$ in the same space.
We introduce below the \emph{diatomic Bellman operator} for the policy evaluation task.

\begin{definition}{\textsc{(Diatomic Bellman operator)}.}
  \label{def:eval}
  Let $\alpha\in (0,1)$.
  Given a stationary Markovian policy $\pi$, the diatomic Bellman operator $\pazocal{T}^\pi_{\alpha}:(\mathbb{R}^2)^{\pazocal{X}\times\pazocal{A}} \rightarrow (\mathbb{R}^2)^{\pazocal{X}\times\pazocal{A}}$ is defined for all $\mathcal{Q}=(Q_1, Q_2):\pazocal{X}\times\pazocal{A}\rightarrow \mathbb{R}^2$ by:
  $\pazocal{T}^\pi_{\alpha} \mathcal{Q} = (Q_1', Q_2')$ such that for each pair $(x,a)$,
  \begin{equation*}
  Q_1'(x,a) = \text{AVaR}^{\text{left}}_{\alpha}((\pazocal{T}^\pi D_{\alpha,\mathcal{Q}})^{(x,a)} ) \quad
  \text{ and } \quad Q_2'(x,a) = \text{AVaR}^{\text{right}}_{1-\alpha}( (\pazocal{T}^\pi D_{\alpha,\mathcal{Q}})^{(x,a)} ) \, .
\end{equation*}
\end{definition}

We point out that the update rule $(Q_1',Q_2')$ in Definition \ref{def:eval} can be expressed explicitly as a function of $(Q_1,Q_2)$ by combining Eq. (\ref{eq:DBO_discrete}) with Lemma \ref{lem:avar}-(i).
More precisely, one shall set $M=2|\pazocal{X}||\pazocal{A}|$ and $v_1\le \dots \le v_M$ the sorted values $(r(x,a,x') + \gamma Q_1(x',a') \,,\, r(x,a,x') + \gamma Q_2(x',a'))_{x',a'}$ with the corresponding probabilities $p_j= \beta P(x'|x,a)\pi(a'|x')$ ($\beta\in\{\alpha, 1-\alpha\}$).
A detailed description of this practical sorting procedure is provided in Algorithm \ref{alg:spe}.

\begin{algorithm}
\caption{\textsc{Sorted Policy Evaluation (SPE)}, single iteration.}\label{alg:spe}

\begin{algorithmic}[1]
\algrenewcommand\algorithmicrequire{\textbf{Parameters:}}
\Require policy $\pi\in\Pi$, number of particles $M = 2|\pazocal{X}||\pazocal{A}|$, level $\alpha\in(0,1)$, $(\alpha_1,\alpha_2)=(\alpha,1-\alpha)$
\algrenewcommand\algorithmicrequire{\textbf{Input:}}
\Require double Q-function $\mathcal{Q}=(Q_1,Q_2)$

\For{each state-action pair $(x,a)\in\pazocal{X}\times\pazocal{A}$}

  \State probability-particle pairs: \begin{equation*}(p_j, v_j)_{j=1}^M \leftarrow (\alpha_i P(x'|x,a)\pi(a'|x'), r(x,a,x')+\gamma Q_i(x',a'))_{(x',a',i)\in\pazocal{X}\times\pazocal{A}\times\{1,2\}} \end{equation*}
  \State particle sorting: $v_{\sigma(1)}\le\dots\le v_{\sigma(M)}$ with $\sigma$ an ``argsort'' permutation
  \State reordering: $(p_j, v_j) \leftarrow (p_{\sigma(j)}, v_{\sigma(j)})$ for $j=1\dots M$
  \State left AVaR: $Q_1'(x,a) \leftarrow \frac{1}{\alpha} \sum_{j=1}^M \max\left( 0\,,\,\min\left( p_j \,,\,\alpha - \sum_{j'\le j-1} p_{j'} \right) \right) \cdot v_j$
  \State right AVaR: $Q_2'(x,a) \leftarrow \frac{1}{1-\alpha} \sum_{j=1}^M \max\left( 0\,,\,\min\left( p_j \,,\,\sum_{j'\le j} p_{j'} - \alpha \right) \right) \cdot v_j$

\EndFor
\Ensure next double Q-function $\pazocal{T}^\pi_{\alpha} \mathcal{Q}=(Q_1',Q_2')$
\end{algorithmic}
\end{algorithm}

\begin{table}
{\begin{tabular}{| c | c | c |}
\hline
\textsc{SPE} Algorithm \ref{alg:spe} & \textsc{Safe/Risky SVI} Algorithm \ref{alg:safe_risky_svi} & Classic value iteration \\
\hline
$\pazocal{O}\left(\left(|\pazocal{X}||\pazocal{A}|\right)^2 \cdot \log\left( |\pazocal{X}||\pazocal{A}| \right) \right)$ & $\pazocal{O}\left(|\pazocal{X}|^2|\pazocal{A}| \cdot \log\left( |\pazocal{X}| \right) \right)$ & $\pazocal{O}\left(|\pazocal{X}|^2|\pazocal{A}|\right)$ \\
\hline
\end{tabular}}
\caption{Time complexity per iteration.
For a deterministic policy, the time complextity of \textsc{SPE} reduces to $\pazocal{O}\left(|\pazocal{X}|^2|\pazocal{A}| \cdot \log\left( |\pazocal{X}| \right) \right)$.
The $|\pazocal{X}||\pazocal{A}|$ sorting steps induce the multiplicative logarithmic terms for Algorithms \ref{alg:spe} and \ref{alg:safe_risky_svi}. Obviously, if the reward function does not depend on the next state (i.e. $r(x,a,\cdot)\equiv r(x,a)$ for all $x,a$), these terms can be suppressed by sorting $(Q_i(x',a'))_{x',a',i}$ just once, then discounting by $\gamma$ and shifting by $r(x,a)$ for any $x,a$.}
\label{tab:time}
\end{table}

Next is a list of some important properties that are satisfied by our new operator $\pazocal{T}^\pi_{\alpha}$.
In particular, the property (iii) and its proof show that $\pazocal{T}^\pi_{\alpha}$ is not affine,
which is in contrast with the classical Bellman operators that are affine.

\begin{proposition}{\textsc{(Properties of $\pazocal{T}^\pi_{\alpha}$)}.}
  \label{prop:properties_eval} Let $\alpha\in (0,1)$, $\pi\in\Pi$
  and denote by $(f',g')\gtrless (f,g)$ the two inequalities $f'\ge f$ and $g'\le g$.
  The following properties are verified.
\begin{enumerate}[(i)]
  \item Monotonicity: if $\mathcal{Q} \ge \widetilde{\mathcal{Q}}$, then $\pazocal{T}^\pi_{\alpha} \mathcal{Q} \ge \pazocal{T}^\pi_{\alpha} \widetilde{\mathcal{Q}}$.
  \item Distributivity: for any $c\in\mathbb{R}$, $\pazocal{T}^\pi_{\alpha} (\mathcal{Q} + c\mathbf{1}) = \pazocal{T}^\pi_{\alpha} \mathcal{Q} + \gamma c \mathbf{1}$.
  \item Concavity/convexity: for $0\le \lambda\le 1$, $\pazocal{T}^\pi_{\alpha} (\lambda\mathcal{Q} + (1-\lambda) \widetilde{\mathcal{Q}} ) \gtrless \lambda \pazocal{T}^\pi_{\alpha} \mathcal{Q} + (1-\lambda) \pazocal{T}^\pi_{\alpha} \widetilde{\mathcal{Q}}$.
  \item $\gamma$-Contraction in sup norm: $|| \pazocal{T}^\pi_{\alpha} \mathcal{Q} - \pazocal{T}^\pi_{\alpha} \widetilde{\mathcal{Q}} ||_\infty \le \gamma || \mathcal{Q} - \widetilde{\mathcal{Q}} ||_\infty$.
  \item Fixed point: there exists a unique fixed point $\mathcal{Q}^\pi = \pazocal{T}^\pi_{\alpha} \mathcal{Q}^\pi$,
  with $\mathcal{Q}^\pi=(Q_1^\pi, Q_2^\pi)$.
  \item Averaging property: $\alpha Q_1^\pi(x,a) + (1-\alpha) Q_2^\pi(x,a) = Q^\pi(x,a)$.
  \item Relative order: $Q_1^\pi(x,a) \le Q^\pi(x,a) \le Q_2^\pi(x,a)$.
\end{enumerate}
\end{proposition}

\begin{proof}
\noindent (i) Monotonicity. $\mathcal{Q}=(Q_1,Q_2) \ge \widetilde{\mathcal{Q}}=(\widetilde{Q}_1,\widetilde{Q}_2)$
means that for all $(x,a)$, $Q_1(x,a)\ge \widetilde{Q}_1(x,a)$ and $Q_2(x,a)\ge \widetilde{Q}_2(x,a)$.
The monotonicity property follows from combining Eq. (\ref{eq:DBO_discrete}) with the dual representation Lemma \ref{lem:avar}-(ii).

\noindent (ii) Distributivity. Fix a pair $(x,a)$ and $c\in\mathbb{R}$.
The quantile function of $(\pazocal{T}^\pi D_{\alpha, \mathcal{Q}+c\mathbf{1}})^{(x,a)}$ is obtained
by shifting that of $(\pazocal{T}^\pi D_{\alpha, \mathcal{Q}})^{(x,a)}$ by $\gamma c$.
The result follows by linearity of the Lebesgue integral.

\noindent (iii) Concavity/convexity. Denoting $(Q_1', Q_2')=\pazocal{T}^\pi_{\alpha} \mathcal{Q}$, it holds from Lemma \ref{lem:avar}-(ii) that
\begin{equation*}
  Q_1'(x,a) = \text{AVaR}^{\text{left}}_{\alpha}((\pazocal{T}^\pi D_{\alpha,\mathcal{Q}})^{(x,a)} )
\end{equation*}
is a concave piecewise linear function of $\mathcal{Q}$. For the same reason, $Q_2'(x,a)$ is a convex piecewise linear function of $\mathcal{Q}$.

\noindent (iv) Contraction.
Given that $\pazocal{T}^\pi$ is a $\gamma$-contraction in $\widetilde{W}_\infty$ (Lemma 3 in \citealp{bellemare2017distributional}),
it is enough to prove that the $W_2$-projection is a non-expansion in $W_\infty$.
Let $(\nu,\nu')\in\mathcal{P}_b(\mathbb{R})^2$ with respective CDFs $F, G$.
Then,
\begin{multline*}
  W_\infty\left( \alpha \delta_{\text{AVaR}^{\text{left}}_{\alpha}(\nu)} + (1-\alpha)\delta_{\text{AVaR}^{\text{right}}_{1-\alpha}(\nu)} \, , \, \alpha \delta_{\text{AVaR}^{\text{left}}_{\alpha}(\nu')} + (1-\alpha)\delta_{\text{AVaR}^{\text{right}}_{1-\alpha}(\nu')} \right) \\
  = \max\left( \left|\text{AVaR}^{\text{left}}_{\alpha}(\nu) - \text{AVaR}^{\text{left}}_{\alpha}(\nu')\right| \, , \, \left| \text{AVaR}^{\text{right}}_{1-\alpha}(\nu) - \text{AVaR}^{\text{right}}_{1-\alpha}(\nu') \right| \right) \\
= \max\left( \frac{1}{\alpha} \left| \int_{\tau=0}^\alpha (F^{-1}(\tau) - G^{-1}(\tau)) d\tau \right| \, , \, \frac{1}{1-\alpha} \left| \int_{\tau=\alpha}^1 (F^{-1}(\tau) - G^{-1}(\tau)) d\tau \right| \right) \\
\le \sup_{\tau\in(0, 1)} |F^{-1}(\tau) - G^{-1}(\tau)| = W_\infty(\nu, \nu')  ,
\end{multline*}
where we used a triangular inequality in the last step.

\noindent (v) Fixed point.
Consequence of (iii) combined with Banach's fixed point theorem and the completeness of $\mathbb{R}^{2|\pazocal{X}||\pazocal{A}|}$.

\noindent (vi) Average.
By Lemma \ref{lem:quantile_expectation}: for any distribution $\nu\in\mathcal{P}_b(\mathbb{R})$ with CDF $F$,
\begin{equation*}
  \alpha \text{AVaR}^{\text{left}}_\alpha(\nu) + (1-\alpha) \text{AVaR}^{\text{right}}_{1-\alpha}(\nu)
  = \int_{\tau=0}^1 F^{-1}(\tau) d\tau
  = \mathbb{E}_{Y\sim \nu}[Y] .
\end{equation*}
It follows for $\nu=(\pazocal{T}^\pi D_{\alpha, \mathcal{Q}^\pi})^{(x,a)}$ that $\alpha Q_1^\pi + (1-\alpha) Q_2^\pi$ solves the classical Bellman equation.

\noindent (vii) Relative order.
As quantile functions are always non-decreasing,
\begin{equation*}
\text{AVaR}^{\text{left}}_\alpha(\nu) \le \text{AVaR}^{\text{right}}_\beta(\nu) \quad \text{ for any } (\alpha,\beta)\in (0, 1]^2 \,,
\end{equation*}
from which the proof follows for $\nu=(\pazocal{T}^\pi D_{\alpha, \mathcal{Q}^\pi})^{(x,a)}$ and $\beta=1-\alpha$ (and from Lemma \ref{lem:quantile_expectation}, the expected value of $\nu$ is equal to $\text{AVaR}^{\text{left}}_1(\nu)=\text{AVaR}^{\text{right}}_1(\nu)$).
\end{proof}

In brief, our approach comes with \emph{two state-action value functions} $Q_1^\pi(x,a)$ and $Q_2^\pi(x,a)$,
instead of just $Q^\pi(x,a)$.
Indeed from Proposition \ref{prop:properties_eval}-(v), we know that our $W_2$-projected operator has a unique fixed point $\mathcal{Q}^\pi=(Q_1^\pi, Q_2^\pi)$.
Plus, properties such as (vi) and (vii) indicate
that this method extends the expected case.
Still, so far we have not yet provided any meaningful interpretation of $Q_1^\pi$ and $Q_2^\pi$.
The next section shows that these quantities are
related to some notion of robustness induced by the policy $\pi$.
By analogy with the AVaR, we call $Q_1^\pi(x,a)$ the \emph{left Bellman average value-at-risk} (left BAVaR)
and $Q_2^\pi(x,a)$ the \emph{right Bellman average value-at-risk} (right BAVaR),
for a given risk level $\alpha \in (0,1)$.

\section{Robustness and risk-awareness}
\label{sec:interpretations}

The purpose of this section is to interpret the left BAVaR $Q_1^\pi(x,a)$ and the right BAVaR $Q_2^\pi(x,a)$
from the double perspective of robustness and risk-measure theory.
\begin{itemize}
\item Our first interpretation bridges the gap with the robust MDP framework:
$Q_1^\pi(x,a)$ and $Q_2^\pi(x,a)$ are respectively ``worst-case'' and ``best-case'' $Q$-functions,
in a latent augmented MDP with twice more states.
To the best of our knowledge, this constitutes the first formal link between the distributional Bellman operator and robust MDPs.
\item Our second point states that $-(1-\gamma)Q_1^\pi(x,a)$ and $(1-\gamma)Q_2^\pi(x,a)$ are both \emph{coherent risk measures}.
\end{itemize}

\subsection{Robust MDP with double state space}

The purpose of this \emph{first interpretation} is to unveil the link existing between our distributional approach and robust MDPs, in the special case of
\emph{$\alpha$-coherent policies}.
\begin{definition}{\textsc{($\alpha$-coherent policy)}.}
  \label{def:unipol}
For $\alpha\in (0,1)$, we define the set $\Pi_\alpha\subseteq\Pi$ of $\alpha$-coherent policies $\pi$ that satisfy the
following condition:
\begin{equation*}
  \forall (x,a,b)\in\pazocal{X}\times\pazocal{A}^2 \quad ,
  \quad (a,b)\in\text{Support}(\pi(\cdot|x))^2 \Longrightarrow \mathcal{Q}^\pi(x,a)=\mathcal{Q}^\pi(x,b)=:(V^\pi_1(x),V^\pi_2(x)) \,.
\end{equation*}
\end{definition}

All deterministic policies belong to $\Pi_\alpha$, as well as the \emph{safe/risky policies} that are derived in section \ref{sec:safe_risky}.
For such $\alpha$-coherent policies, the BAVaR factorizes to a robust MDP formulation.

\paragraph{Augmented state space.}
Let $\pi\in\Pi_\alpha$ be an $\alpha$-coherent policy. We want to show that $V^\pi_1$ and $V^\pi_2$ can respectively be
interpreted as worst-case and best-case value functions in an augmented MDP, where each state $x\in\pazocal{X}$ is
split into two distinct substates $\underline{x}$ and $\overline{x}$.
We denote the augmented state space by:
\begin{equation}
\boldsymbol{X} = \bigcup_{x\in\pazocal{X}} \{\underline{x}, \overline{x}\} ,
\end{equation}
which has double size $|\boldsymbol{X}|=2|\pazocal{X}|$.
We consider a decision-maker that only observes his current state $x$ in the original state space, not the latent substate (either $\underline{x}$ or $\overline{x}$) in the augmented one. For that reason, it is natural to extend the reward function $r$ and the policy $\pi$ to the augmented state space without distinguishing the substates: for any transition $(x,a,x')\in\pazocal{X}\times\pazocal{A}\times\pazocal{X}$, and corresponding substates $s\in\{\underline{x},\overline{x}\}$, $s'\in\{\underline{x'},\overline{x'}\}$,
\begin{equation}
\label{eq:extended_rpi}
r(s,a,s') = r(x,a,x') \quad \text{ and } \quad
\pi(\cdot|s) = \pi(\cdot|x) \,.
\end{equation}

\paragraph{Refined dynamics.}
In the augmented MDP, we constrain the transition probabilities to be consistent with the transition kernel $P$ of the original MDP: this characterizes the following \emph{dichotomous uncertainty set}.

\begin{definition}{\textsc{(Dichotomous uncertainty set)}.}
\label{def:uncertainty_set}
Given the original transition kernel $P$ and $\alpha\in (0,1)$, the dichotomous uncertainty set $\Upsilon_{\alpha}$ is the set of transition kernels $\boldsymbol{P}$ (in the augmented MDP with double state space $\boldsymbol{X}$) verifying: for all $x,a,x'$,
\begin{equation*}
\begin{cases}
  & \alpha \boldsymbol{P}(\underline{x'} | \underline{x}, a) + (1-\alpha) \boldsymbol{P}(\underline{x'} | \overline{x}, a) = \alpha P(x'|x,a) \\
  & \alpha \boldsymbol{P}(\overline{x'} | \underline{x}, a) + (1-\alpha) \boldsymbol{P}(\overline{x'} | \overline{x}, a) = (1-\alpha) P(x'|x,a) \\
  & \boldsymbol{P}(\underline{x'} | \underline{x}, a) \ge \frac{\alpha}{1-\alpha} \boldsymbol{P}(\overline{x'} | \underline{x}, a) \, .
\end{cases}
\end{equation*}
\end{definition}

The inequality constraint in Definition \ref{def:uncertainty_set} ensures that, starting from substate $\underline{x}$,
the next substates with the same \emph{mode} $\underline{x'}$ are visited in priority compared to those with different mode $\overline{x'}$.
All together, these constraints also imply the symmetric inequality: $\boldsymbol{P}(\overline{x'} | \overline{x}, a) \ge \frac{1-\alpha}{\alpha} \boldsymbol{P}(\underline{x'} | \overline{x}, a)$.
Interestingly, our new uncertainty set $\Upsilon_{\alpha}$ does not fulfil any of the rectangularity assumptions that have been considered in the literature (see subsection \ref{subsec:robust_mdps}).
In Definition \ref{def:uncertainty_set}, we highlight that $\Upsilon_{\alpha} \neq \Upsilon_{1-\alpha}$. Still, the two sets are symmetric: $\Upsilon_{1-\alpha}$ is obtained from $\Upsilon_{\alpha}$ by permuting the roles of $\underline{x}$ and $\overline{x}$ for each $x\in\pazocal{X}$.
We are now ready to state our main result.

\begin{theorem}{\textsc{(Robust MDP interpretation)}.}
\label{th:DRL_robust}
Let $\alpha\in (0,1)$ and $\pi\in\Pi_\alpha$ be an $\alpha$-coherent policy.
Let $V^\pi_{\boldsymbol{P}}$ be the value function of $\pi$ in the augmented MDP (see Eq. (\ref{eq:extended_rpi})) with transition kernel $\boldsymbol{P}$. Then, $V^\pi_1$ and $V^\pi_2$ are respectively worst-case and best-case value functions:
\begin{equation*}
\forall x\in\pazocal{X}, \qquad
V^\pi_1(x) = \inf_{\boldsymbol{P}\in\Upsilon_{\alpha}} V^{\pi}_{\boldsymbol{P}}(\underline{x})
\quad \text{ and } \quad
V^\pi_2(x) = \sup_{\boldsymbol{P}\in\Upsilon_{\alpha}} V^{\pi}_{\boldsymbol{P}}(\overline{x}) ,
\end{equation*}
where the infimum and supremum are attained at the same kernel(s) $\boldsymbol{P}^\star$.
\end{theorem}

\begin{proof}
Fix $(x,a)$ and denote by $F_{x,a}$ the CDF of $(\pazocal{T}^\pi D_{\alpha,\mathcal{Q}^\pi})^{(x,a)}$.
Using the dual representation of the left AVaR in Lemma \ref{lem:avar}-(ii),
\begin{equation*}
\frac{1}{\alpha} \int_{\tau=0}^\alpha F_{x,a}^{-1}(\tau)d\tau
= \frac{1}{\alpha} \inf_{\lambda_1,\lambda_2} \sum_{x',a'} \lambda_1(x',a') (r(x, a, x') + \gamma Q_1^\pi(x',a')) + \lambda_2(x',a') (r(x, a, x') + \gamma Q_2^\pi(x',a')) ,
\end{equation*}
where the infimum ranges over functions $(\lambda_1,\lambda_2):\pazocal{X}\times\pazocal{A}\rightarrow [0,1]^2$ such that for all $(x',a')$,
\begin{itemize}
  \item $0\le \lambda_1(x',a') \le \alpha P(x'|x,a)\pi(a'|x')$ ,
  \item $0\le \lambda_2(x',a') \le (1-\alpha) P(x'|x,a)\pi(a'|x')$ ,
  \item $\sum_{x',a'} \lambda_1(x',a')+\lambda_2(x',a')=\alpha$ .
\end{itemize}
Under the additional assumption $\pi\in\Pi_\alpha$, any $Q_1^\pi(x',a')$ (resp.~$Q_2^\pi(x',a')$) with $a'\in\text{Support}(\pi(\cdot|x'))$ is equal to $V^\pi_1(x')$ (resp.~$V^\pi_2(x')$), which simplifies the expression to:
\begin{multline}
  \label{eq:inf_lambda}
\frac{1}{\alpha} \int_{\tau=0}^\alpha F_{x,a}^{-1}(\tau)d\tau
= \inf_{\lambda_1,\lambda_2} \sum_{x'} \underbrace{\left[\frac{1}{\alpha}\sum_{a'} \lambda_1(x',a')\right]}_{\boldsymbol{P}(\underline{x'}|\underline{x},a)} (r(x, a, x') + \gamma V^\pi_1(x')) \\
+ \underbrace{\left[\frac{1}{\alpha}\sum_{a'} \lambda_2(x',a')\right]}_{\boldsymbol{P}(\overline{x'}|\underline{x},a)} (r(x, a, x') + \gamma V^\pi_2(x')).
\end{multline}
Symmetrically for the right AVaR,
\begin{multline}
  \label{eq:sup_lambda}
\frac{1}{1-\alpha} \int_{\tau=\alpha}^1 F_{x,a}^{-1}(\tau)d\tau
= \sup_{\lambda_1,\lambda_2} \sum_{x'} \underbrace{\left[\frac{1}{1-\alpha}\sum_{a'} \alpha P(x'|x,a)\pi(a'|x')-\lambda_1(x',a')\right]}_{\boldsymbol{P}(\underline{x'}|\overline{x},a)} (r(x, a, x') + \gamma V^\pi_1(x')) \\
+ \underbrace{\left[\frac{1}{1-\alpha}\sum_{a'} (1-\alpha) P(x'|x,a)\pi(a'|x')- \lambda_2(x',a')\right]}_{\boldsymbol{P}(\overline{x'}|\overline{x},a)} (r(x, a, x') + \gamma V^\pi_2(x')).
\end{multline}
Hence,
\begin{equation*}
  V^\pi_1(x)
  = \inf_{\boldsymbol{P}\in \Upsilon_{\alpha}} \sum_a \pi(a|x) \sum_{x'} \boldsymbol{P}(\underline{x'}|\underline{x},a) (r(x, a, x') + \gamma V^\pi_1(x')) + \boldsymbol{P}(\overline{x'}|\underline{x},a) (r(x, a, x') + \gamma V^\pi_2(x'))
\end{equation*}
  and
  \begin{equation*}
  V^\pi_2(x)
  = \sup_{\boldsymbol{P}\in \Upsilon_{\alpha}} \sum_a \pi(a|x) \sum_{x'} \boldsymbol{P}(\underline{x'}|\overline{x},a) (r(x, a, x') + \gamma V^\pi_1(x')) + \boldsymbol{P}(\overline{x'}|\overline{x},a) (r(x, a, x') + \gamma V^\pi_2(x')).
\end{equation*}
From equations (\ref{eq:inf_lambda}) and (\ref{eq:sup_lambda}), the infimum and the supremum are clearly attained at the same $(\lambda_1^\star,\lambda_2^\star)$ verifying $\lambda_1^\star \ge \frac{\alpha}{1-\alpha} \lambda_2^\star$ because $V^\pi_1 \le V^\pi_2$.
Denoting by $T^\pi_{\boldsymbol{P}}$ the non-distributional Bellman operator in the augmented MDP with kernel $\boldsymbol{P}$, we just proved that
\begin{equation*}
  V^\pi_1(x)
  = V^\pi_{\boldsymbol{P}^\star}(\underline{x})
  = \inf_{\boldsymbol{P}\in \Upsilon_{\alpha}} ( T^\pi_{\boldsymbol{P}} V^\pi_{\boldsymbol{P}^\star} )(\underline{x}) \quad \text{ and } \quad
  V^\pi_2(x)
  = V^\pi_{\boldsymbol{P}^\star}(\overline{x})
  = \sup_{\boldsymbol{P}\in \Upsilon_{\alpha}} ( T^\pi_{\boldsymbol{P}} V^\pi_{\boldsymbol{P}^\star} )(\overline{x}) \, ,
\end{equation*}
where $\boldsymbol{P}^\star\in \Upsilon_{\alpha}$ is characterized by $(\lambda_1^\star,\lambda_2^\star)$.

To conclude the proof, it remains to prove by induction that for any $k\in\mathbb{N}$, the following properties hold
for all $x\in\pazocal{X}$:
\begin{enumerate}[(a)]
  \item averaging property:
  \begin{equation*}
    \forall \boldsymbol{P}\in \Upsilon_{\alpha} \  , \quad
  \alpha ( (T^\pi_{\boldsymbol{P}})^k V^\pi_{\boldsymbol{P}^\star} )(\underline{x}) + (1-\alpha) ( (T^\pi_{\boldsymbol{P}})^k V^\pi_{\boldsymbol{P}^\star} )(\overline{x}) = V^\pi(x) \, ,
\end{equation*}
  \item monotonicity:
  \begin{equation*}
  \inf_{\boldsymbol{P}\in \Upsilon_{\alpha}} ( (T^\pi_{\boldsymbol{P}})^k V^\pi_{\boldsymbol{P}^\star} )(\underline{x})
  \ge V^\pi_{\boldsymbol{P}^\star}(\underline{x}) \,.
\end{equation*}
\end{enumerate}
The result then follows from taking the limit $k\to+\infty$.

\noindent \emph{Base case} $k=0$.
The monotonicity property (b) trivially holds,
while (a) derives from Proposition \ref{prop:properties_eval}-(vi).

\noindent \emph{Induction step}: assume that the induction hypothesis is true for some $k\ge 0$.

\noindent (a) Averaging property. Let us prove that the averaging property is true for $k+1$:
\begin{multline*}
  \alpha ( (T^\pi_{\boldsymbol{P}})^{k+1} V^\pi_{\boldsymbol{P}^\star} )(\underline{x}) + (1-\alpha) ( (T^\pi_{\boldsymbol{P}})^{k+1} V^\pi_{\boldsymbol{P}^\star} )(\overline{x}) \\
  = \sum_a \pi(a|x) \sum_{x'} \underbrace{ \left( \alpha \boldsymbol{P}(\underline{x'}|\underline{x},a) + (1-\alpha) \boldsymbol{P}(\underline{x'}|\overline{x},a) \right) }_{ \alpha P(x'|x,a) } \cdot \left( r(x,a,x') + \gamma ((T^\pi_{\boldsymbol{P}})^k V^\pi_{\boldsymbol{P}^\star})(\underline{x'}) \right) \\
  + \underbrace{ \left( \alpha \boldsymbol{P}(\overline{x'}|\underline{x},a) + (1-\alpha) \boldsymbol{P}(\overline{x'}|\overline{x},a) \right) }_{ (1-\alpha) P(x'|x,a) } \cdot \left( r(x,a,x') + \gamma ((T^\pi_{\boldsymbol{P}})^k V^\pi_{\boldsymbol{P}^\star})(\overline{x'}) \right) \\
  = \sum_a \pi(a|x) \sum_{x'} P(x'|x,a) \Biggl( r(x,a,x') + \gamma \underbrace{ \left( \alpha ((T^\pi_{\boldsymbol{P}})^k V^\pi_{\boldsymbol{P}^\star})(\underline{x'}) + (1-\alpha) ((T^\pi_{\boldsymbol{P}})^k V^\pi_{\boldsymbol{P}^\star})(\overline{x'}) \right) }_{ V^\pi(x') } \Biggr) = V^\pi(x) .
\end{multline*}

\noindent (b) Monotonicity. We have,
\begin{multline*}
  \inf_{\boldsymbol{P}\in \Upsilon_{\alpha}} ( (T^\pi_{\boldsymbol{P}})^{k+1} V^\pi_{\boldsymbol{P}^\star} )(\underline{x})
  = \inf_{\boldsymbol{P}\in \Upsilon_{\alpha}} \sum_a \pi(a|x) \sum_{x'} \boldsymbol{P}(\underline{x'}|\underline{x},a) \left( r(x,a,x') + \gamma ( (T^\pi_{\boldsymbol{P}})^{k} V^\pi_{\boldsymbol{P}^\star} )(\underline{x'}) \right) \\
  + \boldsymbol{P}(\overline{x'}|\underline{x},a) \Biggl( r(x,a,x') + \gamma \underbrace{ ( (T^\pi_{\boldsymbol{P}})^{k} V^\pi_{\boldsymbol{P}^\star} )(\overline{x'}) }_{ \frac{ V^\pi(x') - \alpha ( (T^\pi_{\boldsymbol{P}})^{k} V^\pi_{\boldsymbol{P}^\star} )(\underline{x'}) }{1-\alpha} } \Biggr) \\
  = \inf_{\boldsymbol{P}\in \Upsilon_{\alpha}} \sum_a \pi(a|x) \sum_{x'} \left( \boldsymbol{P}(\underline{x'}|\underline{x},a) + \boldsymbol{P}(\overline{x'}|\underline{x},a) \right) r(x,a,x') \\
  + \gamma \left( \boldsymbol{P}(\underline{x'}|\underline{x},a) - \frac{\alpha}{1-\alpha} \boldsymbol{P}(\overline{x'}|\underline{x},a) \right) \underbrace{ ( (T^\pi_{\boldsymbol{P}})^{k} V^\pi_{\boldsymbol{P}^\star} )(\underline{x'}) }_{\ge V^\pi_{\boldsymbol{P}^\star}(\underline{x'}) }
  + \frac{\gamma}{1-\alpha} \boldsymbol{P}(\overline{x'}|\underline{x},a) V^\pi(x') \\
  \ge \inf_{\boldsymbol{P}\in \Upsilon_{\alpha}} \sum_a \pi(a|x) \sum_{x'} \boldsymbol{P}(\underline{x'}|\underline{x},a) \left( r(x,a,x') + \gamma V^\pi_{\boldsymbol{P}^\star}(\underline{x'}) \right)
  + \boldsymbol{P}(\overline{x'}|\underline{x},a) \Biggl( r(x,a,x') + \gamma \underbrace{ \frac{ V^\pi(x') - \alpha V^\pi_{\boldsymbol{P}^\star}(\underline{x'}) }{1-\alpha} }_{ V^\pi_{\boldsymbol{P}^\star}(\overline{x'}) } \Biggr) \\
  = \inf_{\boldsymbol{P}\in \Upsilon_{\alpha}} ( T^\pi_{\boldsymbol{P}} V^\pi_{\boldsymbol{P}^\star}) (\underline{x})
  = V^\pi_{\boldsymbol{P}^\star}(\underline{x}) .
\end{multline*}

\end{proof}

Theorem \ref{th:DRL_robust} shows that $V^\pi_1$ and $V^\pi_2$ merge into a value function in a state-augmented MDP
such that $V^\pi_1$ contains worst-case values.
Similar results can be found in the risk-sensitive MDP literature.
In \cite{chow2015risk}, risk-sensitive MDPs with CVaR objective are linked to robust MDPs
where the state space is augmented with a continuous state causing computational issues.
The relation between CVaR and robustness has been also investigated in \cite{osogami2012robustness} at the price of augmenting the state space with the time step variable.
In contrast in Theorem \ref{th:DRL_robust}, we only double the number of states, which makes our approach more tractable.

\paragraph{Characterization of $\boldsymbol{P}^\star$.}
A careful look at the proof of Theorem \ref{th:DRL_robust}, combined with Lemma \ref{lem:avar},
reveals the exact expression of the kernel $\boldsymbol{P}^\star\in\Upsilon_{\alpha}$
attaining both the $\inf$ and the $\sup$:
\begin{equation*}
  \text{for } s\in\{\underline{x},\overline{x}\} , \quad \boldsymbol{P}^\star(\cdot | s, a) = \boldsymbol{P}_{\sigma^\star_{x,a}}(\cdot | s, a) \, ,
\end{equation*}
where $\boldsymbol{P}_{\sigma}$ is defined in Definition \ref{def:Psigma} (in Appendix B)
for a generic permutation $\sigma$.
Here, $\sigma^\star_{x,a}: \mathbf{X} \rightarrow \{1, ..., |\mathbf{X}| \}$ is a permutation that sorts
the $|\mathbf{X}|=2|\pazocal{X}|$ particles
\begin{equation*}
  \left( r(x,a,x') + \gamma \underbrace{ V^\pi_1(x')}_{=: \mathcal{V}^\pi(\underline{x'})} \, , \, r(x,a,x') + \gamma \underbrace{ V^\pi_2(x')}_{=: \mathcal{V}^\pi(\overline{x'})} \right)_{x'\in\pazocal{X}}
\end{equation*}
in non-decreasing order:
\begin{equation*}
  r(x,a,\sigma^{\star-1}_{x,a}(1)) + \gamma \mathcal{V}^\pi(\sigma^{\star-1}_{x,a}(1)) \le \dots \le r(x,a,\sigma^{\star-1}_{x,a}(|\mathbf{X}|)) + \gamma \mathcal{V}^\pi(\sigma^{\star-1}_{x,a}(|\mathbf{X}|)) .
\end{equation*}

\begin{example}
  \label{ex:a2}
  In the MDP in Figure \ref{fig:ex_mdp} (with $\gamma=\frac{1}{2}$), for the deterministic policy $\pi(a_2|x_1)=\pi(a_2|x_2)=1$ and level $\alpha=\frac{1}{2}$,
  \begin{equation*}
    \begin{cases}
    & V_1^\pi(x_1) = 1.5 \\
    & V_2^\pi(x_1) = 2.5 \\
    & V_1^\pi(x_2) = 3.5 \\
    & V_2^\pi(x_2) = 4.5
  \end{cases}
  \quad \text{ and } \quad
    \begin{cases}
    & \boldsymbol{P}^\star( \underline{x_1} | \underline{x_1}, a_2 ) = \boldsymbol{P}^\star( \overline{x_1} | \underline{x_1}, a_2 ) = 0.5 \\
    & \boldsymbol{P}^\star( \underline{x_2} | \overline{x_1}, a_2 ) = \boldsymbol{P}^\star( \overline{x_2} | \overline{x_1}, a_2 ) = 0.5 \\
    & \boldsymbol{P}^\star( \underline{x_1} | \underline{x_2}, a_2 ) = \boldsymbol{P}^\star( \overline{x_1} | \underline{x_2}, a_2 ) = 0.5 \\
    & \boldsymbol{P}^\star( \underline{x_2} | \overline{x_2}, a_2 ) = \boldsymbol{P}^\star( \overline{x_2} | \overline{x_2}, a_2 ) = 0.5 \quad ,
  \end{cases}
  \end{equation*}
  and $\boldsymbol{P}^\star(\cdot|\cdot,a_1)$ can be chosen arbitrarily as long as $\boldsymbol{P}^\star\in\Upsilon_\alpha$.
\end{example}


\paragraph{Bridging the BAVaR-AVaR gap.}
In general, the BAVaRs are not equal to the AVaRs of the distributional return.
A consequence of Theorem \ref{th:DRL_robust} is that the BAVaRs are more concentrated than the AVaRs around their common mean, namely the value function.

\begin{corollary}{\textsc{(BAVaR vs.~AVaR)}.}
  \label{cor:bavar_avar}
  Consider $V_1^\pi(x)$ and $V_2^\pi(x)$ for a given level $\alpha\in(0,1)$, an $\alpha$-coherent policy
$\pi\in\Pi_\alpha$ and a state $x\in\pazocal{X}$. Then for any action $a\in\text{Support}(\pi(\cdot|x))$,
  \begin{equation*}
  \text{AVaR}^{\text{left}}_{\alpha}(\mu_\pi^{(x,a)}) \le V_1^\pi(x) \quad \text{ and } \quad V_2^\pi(x) \le \text{AVaR}^{\text{right}}_{1-\alpha}(\mu_\pi^{(x,a)})\,.
\end{equation*}
\end{corollary}

\begin{proof}
From \cite{follmer2008convex}, the left AVaR at level $\alpha$ of the distribution
$\mu_\pi^{(x,a)}$ admits the dual expression:
\begin{equation}
  \label{eq:full_avar_dual}
  \text{AVaR}^{\text{left}}_{\alpha}(\mu_\pi^{(x,a)})
  = \inf_{\nu \ll \mu_\pi^{(x,a)} : \frac{d\nu}{d\mu_\pi^{(x,a)}} \le \frac{1}{\alpha}} \mathbb{E}_{Z \sim \nu}[Z] .
\end{equation}
On the other side, from Theorem \ref{th:DRL_robust}, $V_1^\pi(x)$ also writes as an infimum of expected values. However, the distributions $\nu$ need to satisfy more constraints than in Eq. (\ref{eq:full_avar_dual}), due to the geometry of the uncertainty set $\Upsilon_{\alpha}$.
Indeed, with evident notations,
\begin{equation*}
  V_1^\pi(x)
  = \inf_{\boldsymbol{P}\in\Upsilon_{\alpha}} \mathbb{E}_{Z \sim \mu_{\pi,\boldsymbol{P}}^{(\underline{x},a)}}[Z] ,
\end{equation*}
where the fixed point distributions in the augmented MDP verify:
\begin{equation*}
\forall \boldsymbol{P}\in\Upsilon_{\alpha} , \quad \alpha \mu_{\pi,\boldsymbol{P}}^{(\underline{x},a)}+(1-\alpha) \mu_{\pi,\boldsymbol{P}}^{(\overline{x},a)} = \mu_\pi^{(x,a)} .
\end{equation*}
Hence, $\mu_{\pi,\boldsymbol{P}}^{(\underline{x},a)}$ satisfies the constraints in Eq. (\ref{eq:full_avar_dual}):
\begin{equation*}
  \mu_{\pi,\boldsymbol{P}}^{(\underline{x},a)} \ll \mu_\pi^{(x,a)} \quad \text{ and } \quad \frac{d\mu_{\pi,\boldsymbol{P}}^{(\underline{x},a)}}{d\mu_\pi^{(x,a)}} \le \frac{1}{\alpha} ,
\end{equation*}
which implies $\text{AVaR}^{\text{left}}_{\alpha}(\mu_\pi^{(x,a)}) \le V_1^\pi(x)$.
The other half of the proof is similar, with $\inf$ replaced by $\sup$, and $\alpha$ by $1-\alpha$.
\end{proof}

We recall that the whole approach developed in this paper relies on the
successive applications of the distributional Bellman operator (DBO) $\pazocal{T}^\pi$ followed by the $W_2$-projection.
Similarly, one could derive $k$-step operators obtained by applying $k$ times
the DBO (instead of just once) before projecting.
Obviously, by taking $k$ arbitrarily large, one can make the gap
between the resulting ``$k$-step BAVaRs'' and the true AVaRs arbitrarily small.

\subsection{The BAVaR coherent risk measure}
\label{subsec:bavar}

We now expose our \emph{second interpretation} stating that
the negative left BAVaR and the right BAVaR are coherent risk measures.
This claim is good news: indeed,
a risk measure is termed coherent if it satisfies
a set of properties that are desirable
in a wide range of applications including financial ones \citep{artzner1999coherent}.

\begin{corollary}{\textsc{(Coherent risk measure)}.}
  \label{cor:bavar}
  Let $\alpha\in (0,1)$, $\pi\in\Pi_\alpha$ and $x\in\pazocal{X}$.
  \begin{enumerate}[(i)]
    \item The quantity $-(1-\gamma) V_1^\pi(x)$, seen as a function of the reward function $r\in\mathbb{R}^{\pazocal{X}\times\pazocal{A}\times\pazocal{X}}$, is a coherent risk measure.
    \item The quantity $(1-\gamma) V_2^\pi(x)$, seen as a function of the negated reward function
$-r\in\mathbb{R}^{\pazocal{X}\times\pazocal{A}\times\pazocal{X}}$, is a coherent risk measure.
  \end{enumerate}
\end{corollary}

\begin{proof}
\noindent (ii) Right BAVaR.
Let us prove that $(1-\gamma)V_2^\pi(x)$ is a coherent risk measure.
Here, $V_2^\pi(x)$ is seen a function of the opposite reward function $-r$.
For clarity, we denote it by $V_2^\pi(x;-r)$ for a given reward function $r\in\mathbb{R}^{\pazocal{X}\times\pazocal{A}\times\pazocal{X}}$.
We need to prove each of the following four properties (see \citealp{artzner1999coherent}).
\begin{enumerate}[(a)]
  \item Translation invariance:
  \begin{equation*}
  \forall \beta \in \mathbb{R} , \quad (1-\gamma) V_2^\pi(x;-r + \beta \mathbf{1}) = (1-\gamma) V_2^\pi(x;-r) - \beta .
  \end{equation*}
  \item Sub-additivity: for all reward functions $r_1,r_2$,
  \begin{equation*}
  (1-\gamma) V_2^\pi(x;-r_1-r_2) \le (1-\gamma) V_2^\pi(x;-r_1) + (1-\gamma) V_2^\pi(x;-r_2) .
\end{equation*}
  \item Positive homogeneity:
  \begin{equation*}
    \forall \beta \ge 0 , \quad (1-\gamma) V_2^\pi(x;-\beta r) = \beta (1-\gamma) V_2^\pi(x;-r) .
  \end{equation*}
  \item Monotonicity:
  \begin{equation*}
  \text{if }\  -r_1 \le -r_2, \  \text{ then } \  (1-\gamma) V_2^\pi(x;-r_1) \ge (1-\gamma) V_2^\pi(x;-r_2) .
\end{equation*}
\end{enumerate}
All four properties easily follow from Theorem \ref{th:DRL_robust}, by writing $V_2^\pi(x)$
as a supremum of value functions.

\noindent (i) Left BAVaR. The proof is similar to (ii), except that we parametrize
$V_1^\pi(x)$ by the reward function, namely $V_1^\pi(x;r)$.
\end{proof}

This result bridges the gap between our distributional point of view in section \ref{sec:avar}
and the concept of coherent risk measure.
The fixed point object $\mathcal{Q}^\pi=(Q_1^\pi,Q_2^\pi)$
now appears as an appealing objective for risk-aware purpose in MDPs.
Equipped with Theorem \ref{th:DRL_robust}, we follow in the next section the robust MDP paradigm: maximize
the worst-case value function $V^\pi_1$ to obtain a safe policy.

\section{Safe or risky control}
\label{sec:safe_risky}
We now turn to describing a dynamic programming framework for risk-aware control problems derived from our
distributional MDP perspective introduced in the previous sections.
By leveraging the robustness insights developed earlier, it makes
sense to either
\begin{itemize}
  \item look for a \emph{safe policy} that maximizes $V_1^\pi$ and minimizes $V_2^\pi$,
  \item or rather for a \emph{risky policy} that minimizes $V_1^\pi$ and maximizes $V_2^\pi$.
\end{itemize}
While there is a degree of symmetry to these tasks, it is easy to see that there are fundamental differences
between them: while risky control aims to solve a relatively straightforward maximization problem, the safe control
objective has a more intricate ``max-min'' nature. In what follows, we study both objectives under a specific
simplifying assumption on the underlying MDP that allows an effective and (relatively) symmetric treatment of both
cases.

\subsection{Breaking the optimality ties}

We provide the dynamic programming toolbox for solving these two control tasks
for a very special type of MDP, that we call \emph{balanced MDP}, where the \emph{expected} returns of all policies are
equal.

\begin{assumption}{\textsc{(Balanced MDP)}.}
  \label{ass:democratic}
A Markov decision process is said ``balanced'' if $\pazocal{A}^*(x)=\pazocal{A}$ for every state $x\in\pazocal{X}$.
In other words, all policies are optimal in terms of their expected return:
\begin{equation*}
\forall \pi\in\Pi \, , \quad Q^\pi = Q^*  
\quad \text{ or equivalently } \quad \forall (x,a)\in\pazocal{X}\times\pazocal{A} \, , \quad Q^*(x,a)=V^*(x) \,.
\end{equation*}
\end{assumption}

Under such assumption, there is a clear \emph{trade-off} between $Q_1^\pi$ and $Q_2^\pi$. Indeed from Proposition \ref{prop:properties_eval}-(vi), their average has to remain constant:
\begin{equation*}
  \forall \pi\in\Pi \, , \quad \alpha Q_1^\pi + (1-\alpha) Q_2^\pi = Q^* \,.
\end{equation*}
Hence, maximizing one of these two quantities necessarily means minimizing the other one.
Of course, any MDP can be reduced to a balanced MDP by 1) first, identifying the set of
optimal actions $\pazocal{A}^*(x)$ in each state $x$ (classic control problem) and 2) then, filtering the action space to only allow optimal actions.

\begin{example}
  \label{ex:democratic}
  The MDP in Figure \ref{fig:ex_mdp} combined with the discout factor $\gamma=0.5$ is a balanced MDP.
  In section \ref{sec:experiments}, we run experiments based on this MDP.
\end{example}

\subsection{Safe and risky operators}

In balanced MDPs, we necessarily have
\begin{equation*}
Q_2^\pi=\frac{Q^*-\alpha Q_1^\pi}{1-\alpha} \quad , \quad \text{ for any policy } \pi \,.
\end{equation*}
In other words, $Q_2^\pi$ is completely characterized by $Q_1^\pi$ (and vice-versa): hence, we can restrict our attention to $Q_1^\pi$ to define our risk-aware operators more concisely.

\begin{definition}{\textsc{(Safe \& risky Bellman operators)}.}
  \label{def:safe_risky_ops}
  Consider a balanced MDP and let $\alpha\in (0,1)$.
  \begin{enumerate}[(i)]
  \item The safe Bellman operator $\pazocal{T}^{\text{safe}}_{\alpha}:\mathbb{R}^{\pazocal{X}\times\pazocal{A}} \rightarrow \mathbb{R}^{\pazocal{X}\times\pazocal{A}}$ is defined for all $Q_1:\pazocal{X}\times\pazocal{A}\rightarrow \mathbb{R}$ by: $\pazocal{T}^{\text{safe}}_{\alpha} Q_1 = Q_1'$ such that for any $x,a$,
  \begin{multline*}
  Q_1'(x,a) = \text{AVaR}^{\text{left}}_{\alpha}\Biggl(
  \sum_{x'} P(x'|x,a) \Biggl( \alpha \delta_{r(x,a,x')+\gamma \max_{a'} Q_1(x',a') }
  + (1-\alpha) \delta_{r(x,a,x')+\gamma \min_{a'} Q_2(x',a') } \Biggr)
  \Biggr) \, ,
  \end{multline*}
  where $Q_2(x',a') = \frac{V^*(x')-\alpha Q_1(x',a')}{1-\alpha}$.
  \item The risky Bellman operator $\pazocal{T}^{\text{risky}}_\alpha:\mathbb{R}^{\pazocal{X}\times\pazocal{A}} \rightarrow \mathbb{R}^{\pazocal{X}\times\pazocal{A}}$ is defined for all $Q_1:\pazocal{X}\times\pazocal{A}\rightarrow \mathbb{R}$ by: $\pazocal{T}^{\text{risky}}_{\alpha} Q_1 = Q_1'$ such that for any $x,a$,
  \begin{multline*}
    Q_1'(x,a) = \text{AVaR}^{\text{left}}_{\alpha}\Biggl(
    \sum_{x'} P(x'|x,a) \Biggl( \alpha \delta_{r(x,a,x')+\gamma \min_{a'} Q_1(x',a') }
    + (1-\alpha) \delta_{r(x,a,x')+\gamma \max_{a'} Q_2(x',a') } \Biggr)
    \Biggr) \, ,
  \end{multline*}
  where $Q_2(x',a') = \frac{V^*(x')-\alpha Q_1(x',a')}{1-\alpha}$.
\end{enumerate}
\end{definition}

As for the policy evaluation operator, these two control operators
can be easily implemented through a sorting step:
see Algorithm \ref{alg:safe_risky_svi}.
Moreover, they satisfy the properties listed below.

\begin{algorithm}
\caption{\textsc{Safe/Risky Sorted Value Iteration (Safe/Risky SVI)}, single iteration.}\label{alg:safe_risky_svi}

\begin{algorithmic}[1]
\algrenewcommand\algorithmicrequire{\textbf{Parameters:}}
\Require $\text{mode}\in\{\text{safe},\text{risky}\}$, optimal value function $V^*$, number of particles $M = 2|\pazocal{X}|$, level $\alpha\in(0,1)$, $(\alpha_1,\alpha_2)=(\alpha,1-\alpha)$
\algrenewcommand\algorithmicrequire{\textbf{Input:}}
\Require Q-function $Q_1$

\If{$\text{mode}=\text{safe}$}

  \For{each state $x\in\pazocal{X}$}
  \State first value function: $V_1(x) \leftarrow \max_{a} Q_1(x,a)$
  \State second value function: $V_2(x) \leftarrow \min_{a} \frac{V^*(x)-\alpha Q_1(x,a)}{1-\alpha}$
  \EndFor

\ElsIf{$\text{mode}=\text{risky}$}

  \For{each state $x\in\pazocal{X}$}
  \State first value function: $V_1(x) \leftarrow \min_{a} Q_1(x,a)$
  \State second value function: $V_2(x) \leftarrow \max_{a} \frac{V^*(x)-\alpha Q_1(x,a)}{1-\alpha}$
  \EndFor

\EndIf

\For{each state-action pair $(x,a)\in\pazocal{X}\times\pazocal{A}$}

  \State probability-particle pairs: \begin{equation*}(p_j, v_j)_{j=1}^M \leftarrow (\alpha_i P(x'|x,a), r(x,a,x')+\gamma V_i(x'))_{(x',i)\in\pazocal{X}\times\{1,2\}} \end{equation*}
  \State particle sorting: $v_{\sigma(1)}\le\dots\le v_{\sigma(M)}$ with $\sigma$ an ``argsort'' permutation
  \State reordering: $(p_j, v_j) \leftarrow (p_{\sigma(j)}, v_{\sigma(j)})$ for $j=1\dots M$
  \State left AVaR: $Q_1'(x,a) \leftarrow \frac{1}{\alpha} \sum_{j=1}^M \max\left( 0\,,\,\min\left( p_j \,,\,\alpha - \sum_{j'\le j-1} p_{j'} \right) \right) \cdot v_j$

\EndFor
\Ensure next Q-function $\pazocal{T}^{\text{mode}}_{\alpha} Q_1=Q_1'$
\end{algorithmic}

\end{algorithm}

\begin{proposition}{\textsc{(Properties of $\pazocal{T}^{\text{safe}}_{\alpha}$ and $\pazocal{T}^{\text{risky}}_{\alpha}$)}.}
  \label{prop:properties_risk}
  Assume a balanced MDP, let $\alpha\in (0,1)$, $\text{mode} \in \{\text{safe}, \text{risky} \}$. The following
properties hold.
\begin{enumerate}[(i)]
  \item Concavity of the risky operator: if $Q_1\le Q^*$ and $\widetilde{Q}_1\le Q^*$, then for any $0\le \lambda\le 1$,
  \begin{equation*}
  \pazocal{T}^{\text{risky}}_{\alpha} (\lambda Q_1 + (1-\lambda) \widetilde{Q}_1 )
  \ge \lambda \pazocal{T}^{\text{risky}}_{\alpha} Q_1 + (1-\lambda) \pazocal{T}^{\text{risky}}_{\alpha} \widetilde{Q}_1 .
\end{equation*}
  \item $\gamma$-Contraction in sup norm: $|| \pazocal{T}^{\text{mode}}_{\alpha} Q_1 - \pazocal{T}^{\text{mode}}_{\alpha} \widetilde{Q}_1 ||_\infty \le \gamma || Q_1 - \widetilde{Q}_1 ||_\infty$.
  \item Fixed point: there exists a unique fixed point $Q_1^{\text{mode}} = \pazocal{T}^{\text{mode}}_{\alpha} Q_1^{\text{mode}}$.
  \item Safe optimality: for all $(x,a)\in\pazocal{X}\times\pazocal{A}$\,,
  \begin{equation*}
    Q_1^{\text{safe}}(x,a) = \sup_{\pi} Q_1^\pi(x,a) \quad \text{ or equivalently } \quad
    Q_2^{\text{safe}}(x,a) = \inf_{\pi} Q_2^\pi(x,a) \, ,
  \end{equation*}
  where $Q_2^{\text{safe}}(x,a) := (V^*(x)-\alpha Q_1^{\text{safe}}(x,a))/(1-\alpha)$.
\item Risky optimality: for all $(x,a)\in\pazocal{X}\times\pazocal{A}$\,,
\begin{equation*}
  Q_1^{\text{risky}}(x,a) = \inf_{\pi} Q_1^\pi(x,a) \quad \text{ or equivalently } \quad
  Q_2^{\text{risky}}(x,a) = \sup_{\pi} Q_2^\pi(x,a) \, ,
\end{equation*}
where $Q_2^{\text{risky}}(x,a) := (V^*(x)-\alpha Q_1^{\text{risky}}(x,a))/(1-\alpha)$.
\end{enumerate}
\end{proposition}

\begin{proof}
\noindent (i) Concavity. The assumption $Q_1\le Q^*$ is equivalent to $Q_1\le Q_2$ with $Q_2(x,a) = \frac{V^*(x)-\alpha Q_1(x,a)}{1-\alpha}$.
Then, given $Q_1'=\pazocal{T}^{\text{risky}}_{\alpha} Q_1$ and using the dual representation of the left AVaR from Lemma \ref{lem:avar}-(ii),
\begin{equation*}
Q_1'(x,a) =
\frac{1}{\alpha} \inf_{\lambda_1,\lambda_2} \sum_{x'} (\lambda_1(x')+\lambda_2(x')) r(x, a, x') + \gamma (\lambda_1(x')-\frac{\alpha}{1-\alpha}\lambda_2(x')) \min_{a'} Q_1(x',a') + \gamma \frac{\lambda_2(x')}{1-\alpha} V^*(x') ,
\end{equation*}
where the infimum is necessarily attained for $\lambda_1 \ge \frac{\alpha}{1-\alpha}\lambda_2$.
We deduce that $Q_1'(x,a)$ is a concave piecewise linear function for $Q_1\le Q^*$.

\noindent (ii) Contraction.
Let us consider the safe case: $\text{mode}=\text{safe}$.
Given a pair $(x,a)$, by the dual representation of the left AVaR in Lemma \ref{lem:avar}-(ii),
\begin{multline*}
( \pazocal{T}^{\text{safe}}_{\alpha} Q_1 )(x,a) = \\
\frac{1}{\alpha} \inf_{\lambda_1,\lambda_2} \sum_{x'} \lambda_1(x') (r(x, a, x') + \gamma \max_{a'} Q_1(x',a')) + \lambda_2(x') (r(x, a, x') + \gamma \min_{a'} \frac{V^*(x')-\alpha Q_1(x',a')}{1-\alpha} ) \\
= \frac{1}{\alpha} \inf_{\lambda_1,\lambda_2} \sum_{x'} (\lambda_1(x')+\lambda_2(x')) r(x, a, x') + \gamma (\lambda_1(x')-\frac{\alpha}{1-\alpha}\lambda_2(x')) \max_{a'} Q_1(x',a') + \gamma \frac{\lambda_2(x')}{1-\alpha} V^*(x') ,
\end{multline*}
where the infimum ranges over functions $(\lambda_1,\lambda_2):\pazocal{X}\rightarrow [0,1]^2$ such that for all $x'$,
\begin{itemize}
  \item $0\le \lambda_1(x') \le \alpha P(x'|x,a)$ ,
  \item $0\le \lambda_2(x') \le (1-\alpha) P(x'|x,a)$ ,
  \item $\sum_{x'} \lambda_1(x')+\lambda_2(x')=\alpha$ .
\end{itemize}
Hence, by successive applications of the triangular inequality,
\begin{equation*}
  \left| ( \pazocal{T}^{\text{safe}}_{\alpha} Q_1 )(x,a) - ( \pazocal{T}^{\text{safe}}_{\alpha} \widetilde{Q}_1 )(x,a) \right|
  \le \frac{\gamma ||Q_1-\widetilde{Q}_1||_\infty}{\alpha} \sup_{\lambda_1,\lambda_2} \sum_{x'} \underbrace{\left|\lambda_1(x')-\frac{\alpha}{1-\alpha} \lambda_2(x') \right|}_{\le \alpha P(x'|x,a)}
  \le \gamma ||Q_1-\widetilde{Q}_1||_\infty .
\end{equation*}
The risky case is analogous.

\noindent (iii) Fixed point.
Consequence of (i) combined with Banach's fixed point theorem and the completeness of $\mathbb{R}^{|\pazocal{X}||\pazocal{A}|}$.

\noindent (iv) Safe optimality.
Fix a policy $\pi$.
Let us prove by induction that for any $k\in\mathbb{N}^*$: for all $(x,a)\in\pazocal{X}\times\pazocal{A}$,
\begin{enumerate}[(a)]
  \item relative order:
  \begin{equation*}
  ( ( \pazocal{T}^{\text{safe}}_{\alpha} )^k Q_1^\pi )(x,a) \le V^*(x) \, ,
\end{equation*}
  \item monotonicity:
  \begin{equation*}
  ( ( \pazocal{T}^{\text{safe}}_{\alpha} )^k Q_1^\pi )(x,a) \ge Q_1^\pi(x,a) \, .
\end{equation*}
\end{enumerate}
Taking the limit $k\to+\infty$ will then conclude the proof.

\noindent \emph{Base case} $k=1$.

\noindent (b) Monotonicity.
As $Q_1^\pi \le Q_2^\pi = (Q^*-\alpha Q_1^\pi)/(1-\alpha)$, then the infimum below
is necessarily attained for $\lambda_1 \ge \frac{\alpha}{1-\alpha}\lambda_2$:
\begin{multline}
  \label{eq:proof_safe_case}
( \pazocal{T}^{\text{safe}}_{\alpha} Q_1^\pi )(x,a) = \\
\frac{1}{\alpha} \inf_{\lambda_1,\lambda_2} \sum_{x'} (\lambda_1(x')+\lambda_2(x')) r(x, a, x') + \gamma (\lambda_1(x')-\frac{\alpha}{1-\alpha}\lambda_2(x')) \max_{a'} Q_1^\pi(x',a') + \gamma \frac{\lambda_2(x')}{1-\alpha} V^*(x') \\
\ge \frac{1}{\alpha} \inf_{\widetilde{\lambda}_1,\widetilde{\lambda}_2} \sum_{x',a'} (\widetilde{\lambda}_1(x',a')+\widetilde{\lambda}_2(x',a')) r(x, a, x') + \gamma (\widetilde{\lambda}_1(x',a')-\frac{\alpha}{1-\alpha}\widetilde{\lambda}_2(x',a')) Q_1^\pi(x',a') + \gamma \frac{\widetilde{\lambda}_2(x',a')}{1-\alpha} V^*(x') \\
= \text{AVaR}^{\text{left}}_{\alpha}((\pazocal{T}^\pi D_{\alpha,\mathcal{Q}^\pi})^{(x,a)} )
= Q_1^\pi(x,a) ,
\end{multline}
where the second infimum ranges over functions $\widetilde{\lambda}_1 \ge \frac{\alpha}{1-\alpha}\widetilde{\lambda}_2$ such that for all $x',a'$,
\begin{itemize}
  \item $0\le \widetilde{\lambda}_1(x',a') \le \alpha P(x'|x,a) \pi(a'|x')$ ,
  \item $0\le \widetilde{\lambda}_2(x',a') \le (1-\alpha) P(x'|x,a) \pi(a'|x')$ ,
  \item $\sum_{x',a'} \widetilde{\lambda}_1(x',a')+\widetilde{\lambda}_2(x',a')=\alpha$ .
\end{itemize}

\noindent (a) Relative order.
From the first infimum in Eq. (\ref{eq:proof_safe_case}) for $\lambda_1 \ge \frac{\alpha}{1-\alpha}\lambda_2$, and using that $Q_1^\pi \le Q^* \equiv V^*$,
\begin{multline*}
  ( \pazocal{T}^{\text{safe}}_{\alpha} Q_1^\pi )(x,a)
  \le \frac{1}{\alpha} \inf_{\lambda_1,\lambda_2} \sum_{x'} (\lambda_1(x')+\lambda_2(x')) r(x, a, x') + \gamma (\lambda_1(x')-\frac{\alpha}{1-\alpha}\lambda_2(x')) V^*(x') + \gamma \frac{\lambda_2(x')}{1-\alpha} V^*(x') \\
  = \frac{1}{\alpha} \inf_{\lambda_1,\lambda_2} \sum_{x'} (\lambda_1(x')+\lambda_2(x')) ( r(x, a, x') + \gamma V^*(x') ) \\
  \le \sum_{x'} P(x'|x,a) ( r(x, a, x') + \gamma V^*(x') ) = Q^*(x,a) = V^*(x) \, ,
\end{multline*}
where the last inequality is obtained by choosing the ``risk-neutral'' weight functions $(\lambda_1,\lambda_2)=(\lambda_1^{\circ},\lambda_2^{\circ})$ defined for all $x'$ by
\begin{equation*}
  \begin{cases}
    & \lambda_1^{\circ}(x') = \alpha^2 P(x'|x,a) \\
    & \lambda_2^{\circ}(x') = \alpha (1-\alpha) P(x'|x,a) \, .
  \end{cases}
\end{equation*}

\noindent \emph{Induction step}: assume that the induction hypothesis is true for some $k\ge 1$.

\noindent (a) Relative order. Let us prove that the inequality holds for $k+1$:
\begin{multline*}
  ( ( \pazocal{T}^{\text{safe}}_{\alpha} )^{k+1} Q_1^\pi )(x,a) = \\
  \frac{1}{\alpha} \inf_{\lambda_1,\lambda_2} \sum_{x'} (\lambda_1(x')+\lambda_2(x')) r(x, a, x') + \gamma (\lambda_1(x')-\frac{\alpha}{1-\alpha}\lambda_2(x')) \underbrace{ \max_{a'} ( ( \pazocal{T}^{\text{safe}}_{\alpha} )^k Q_1^\pi )(x',a') }_{ \le V^*(x') }
  + \gamma \frac{\lambda_2(x')}{1-\alpha} V^*(x') \\
  \le \frac{1}{\alpha} \inf_{\lambda_1,\lambda_2} \sum_{x'} (\lambda_1(x')+\lambda_2(x')) ( r(x, a, x') + \gamma V^*(x') ) \\
  \le \sum_{x'} P(x'|x,a) ( r(x, a, x') + \gamma V^*(x') ) = Q^*(x,a) = V^*(x) \, .
\end{multline*}

\noindent (b) Monotonicity. We have,
\begin{multline*}
  ( ( \pazocal{T}^{\text{safe}}_{\alpha} )^{k+1} Q_1^\pi )(x,a) = \\
  \frac{1}{\alpha} \inf_{\lambda_1,\lambda_2} \sum_{x'} (\lambda_1(x')+\lambda_2(x')) r(x, a, x') + \gamma (\lambda_1(x')-\frac{\alpha}{1-\alpha}\lambda_2(x')) \max_{a'} \underbrace{ ( ( \pazocal{T}^{\text{safe}}_{\alpha} )^k Q_1^\pi )(x',a') }_{ \ge Q_1^\pi(x',a') }
  + \gamma \frac{\lambda_2(x')}{1-\alpha} V^*(x') \\
  \ge ( \pazocal{T}^{\text{safe}}_{\alpha} Q_1^\pi )(x,a) \ge Q_1^\pi(x,a) \, .
\end{multline*}

\noindent (v) Risky optimality.
The proof is similar to the safe case (iii).
\end{proof}

Basically, Proposition \ref{prop:properties_risk} says that the safe and risky Bellman operators
enjoy properties that are similar to the ones verified by the classical Bellman optimality operators.
In particular, their fixed points $Q_1^{\text{safe}}$ and $Q_1^{\text{risky}}$
allow to identify the safest and riskiest actions and policies.

\subsection{Safest and riskiest policies}

From Proposition \ref{prop:properties_risk}, it is natural to define
the set of the safest/riskiest policies as follows:
\begin{multline*}
  \Pi^{\text{safe}}_\alpha := \{ \pi \in\Pi : \mathcal{Q}^\pi = (Q_1^{\text{safe}}, Q_2^{\text{safe}}) \}
\quad \text{ and } \quad  \Pi^{\text{risky}}_{\alpha} := \{ \pi \in\Pi : \mathcal{Q}^\pi = (Q_1^{\text{risky}}, Q_2^{\text{risky}}) \} \,.
\end{multline*}
The following corollary claims that these sets are non-empty and simply characterized by the fixed points:
this is an immediate consequence of Proposition \ref{prop:properties_risk}-(iii).

\begin{corollary}{\textsc{(Safest \& riskiest actions)}.}
  \label{cor:risky_pol}
  Consider a balanced MDP and $\alpha\in (0,1)$.
\begin{enumerate}[(i)]
  \item Safest policies: $\pi\in\Pi^{\text{safe}}_\alpha$ if and only if in each state $x\in\pazocal{X}$,
  \begin{equation*}
    \text{Support}(\pi(\cdot|x))\subseteq \pazocal{A}^{\text{safe}}_\alpha(x)
    := \argmax_{a} Q^{\text{safe}}_1(x,a)
    = \argmin_{a} Q^{\text{safe}}_2(x,a) .
  \end{equation*}
  \item Riskiest policies: $\pi\in\Pi^{\text{risky}}_{\alpha}$ if and only if in each state $x\in\pazocal{X}$,
  \begin{equation*}
    \text{Support}(\pi(\cdot|x))\subseteq \pazocal{A}^{\text{risky}}_\alpha(x)
    := \argmin_{a} Q^{\text{risky}}_1(x,a)
    = \argmax_{a} Q^{\text{risky}}_2(x,a) .
  \end{equation*}
\end{enumerate}
\end{corollary}

Thereby in a balanced MDP, the safest (resp.~riskiest) actions/policies can be identified
by computing the fixed point of the safe (resp.~risky) Bellman operator.
This is analogous to the set of optimal actions $\pazocal{A}^*(x)=\argmax_a Q^*(x,a)$ in the classic control problem.
From Corollary \ref{cor:risky_pol}, there always exist safe and risky policies that are \emph{deterministic}.
Plus, notice that all these safest and riskiest policies are $\alpha$-coherent ($\Pi^{\text{safe}}_\alpha \cup
\Pi^{\text{risky}}_{\alpha} \subseteq \Pi_\alpha$),
which means they come with the robust MDP interpretation of Theorem \ref{th:DRL_robust}.


\section{Numerical illustrations}
\label{sec:experiments}

In this section, we test our Algorithms \ref{alg:spe} and \ref{alg:safe_risky_svi} in a practical example,
for the risk level $\alpha = \frac{1}{2}$.
We run our experiments in the two-states two-actions MDP from Figure \ref{fig:ex_mdp} combined with
the discount factor $\gamma=\frac{1}{2}$, which constitutes a balanced MDP.
Indeed, the Q-function
\begin{equation*}
  \begin{cases}
    & Q^*(x_1,a_1) = Q^*(x_1,a_2) = 2 \\
    & Q^*(x_2,a_1) = Q^*(x_2,a_2) = 4 \, ,
  \end{cases}
\end{equation*}
solves the Bellman optimality equation
\begin{equation*}
  \begin{cases}
    & Q^*(x_1,a_1) = 1 + \frac{1}{2} \max_a Q^*(x_1,a) \\
    & Q^*(x_1,a_2) = \frac{1}{2} + \frac{1}{2} \left( \frac{1}{2} \max_a Q^*(x_1,a) + \frac{1}{2} \max_a Q^*(x_2,a) \right) \\
    & Q^*(x_2,a_1) = 2 + \frac{1}{2} \max_a Q^*(x_2,a) \\
    & Q^*(x_2,a_2) = \frac{5}{2} + \frac{1}{2} \left( \frac{1}{2} \max_a Q^*(x_1,a) + \frac{1}{2} \max_a Q^*(x_2,a) \right) \, .
  \end{cases}
\end{equation*}

\subsection{Evaluation of a policy}

Consider the policy $\pi$ picking uniformly at random the two actions in any of the two states:
\begin{equation*}
  \pi(a | x) = \frac{1}{2} \quad \text{ for all } (x,a) \in \{x_1,x_2\}\times\{a_1,a_2\} .
\end{equation*}
For the policy evaluation task, we run $20$ iterations of the \textsc{SPE} algorithm, starting from
$Q_1(x,a)$ and $Q_2(x,a)$ initialized (arbitrarily) at zero.
Figure \ref{fig:pol_eval} displays the plots across iterations, showing quick convergence to the fixed point $(Q_1^\pi,Q_2^\pi)$.

\begin{figure}

\centering
\begin{subfigure}{.49\textwidth}
  \centering
  \includegraphics[width=1.0\linewidth]{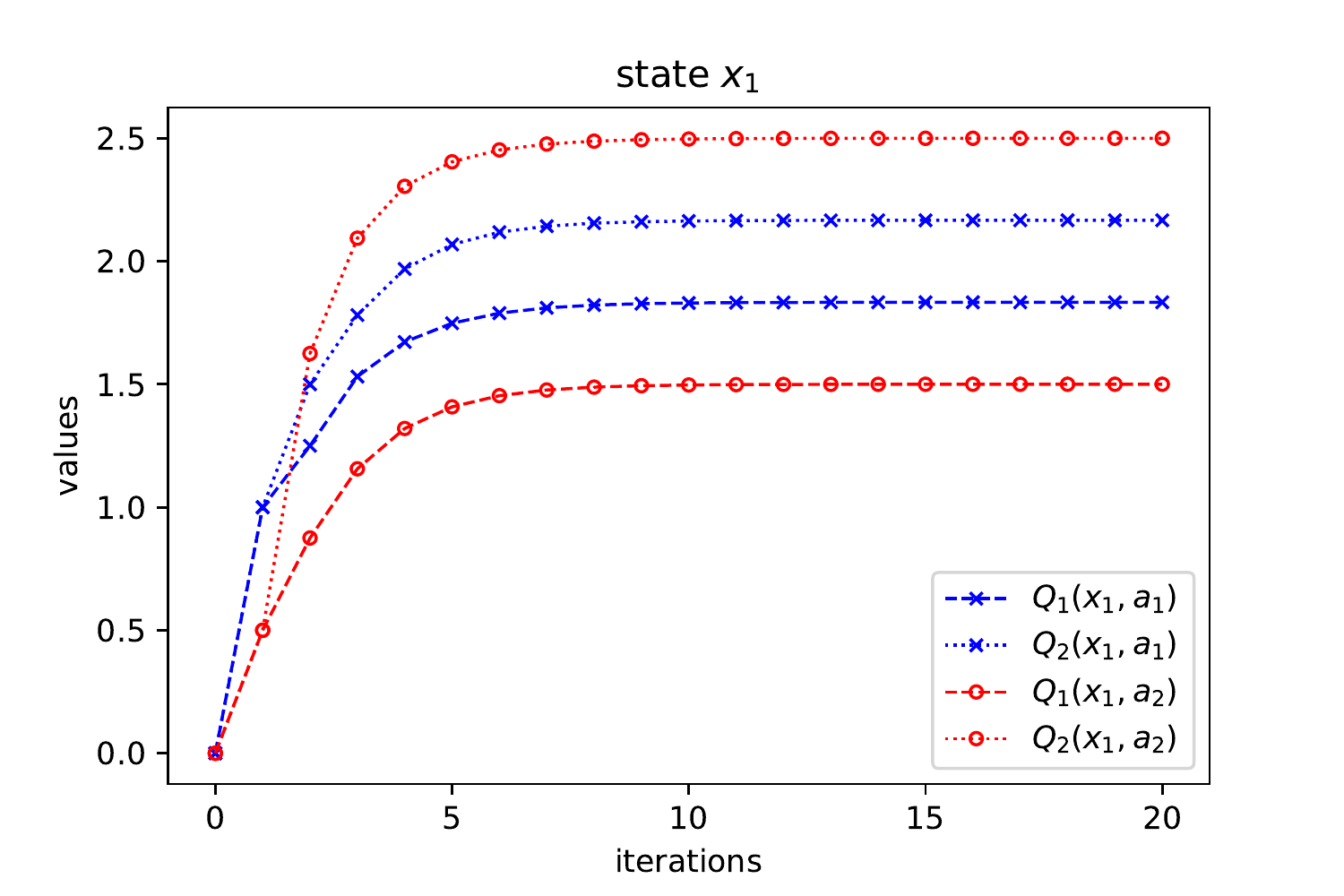}
\end{subfigure}%
\begin{subfigure}{.49\textwidth}
  \centering
  \includegraphics[width=1.0\linewidth]{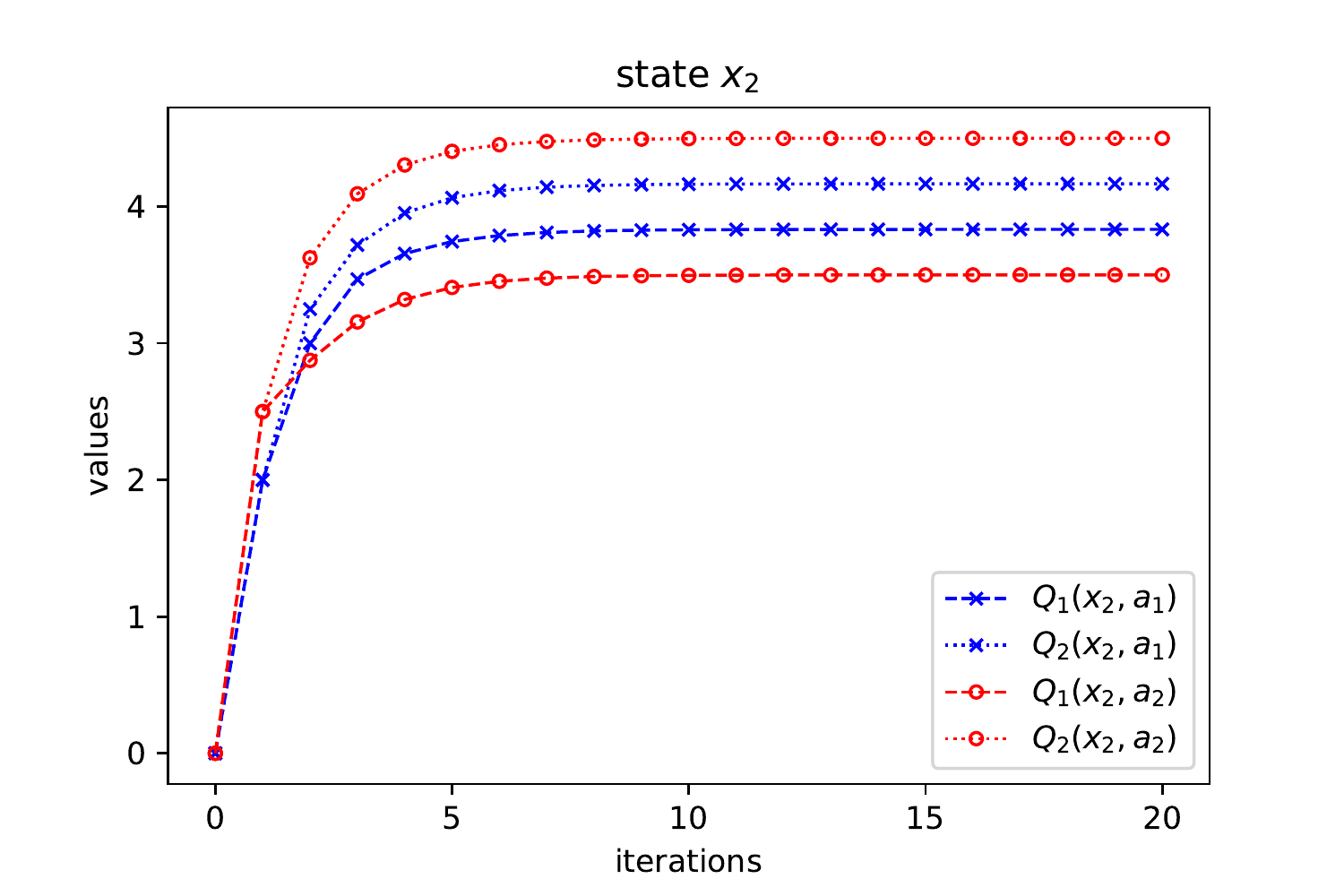}
\end{subfigure}

\caption{Evaluation of the policy $\pi$ by successive iterations of the \textsc{SPE} algorithm \ref{alg:spe}.}
\label{fig:pol_eval}
\end{figure}

\subsection{Finding safe and risky policies}

For the safe (resp.~risky) control task, we run $20$ iterations of the \textsc{Safe SVI} (resp.~\textsc{Risky SVI}) algorithm, starting from
$Q_1(x,a)$ initialized at zero.
In both cases, we see quick convergence to the fixed points.
In Figure \ref{fig:safe}, the safe fixed point satisfies $Q_1^{\text{safe}}(\cdot,a_1)=Q_2^{\text{safe}}(\cdot,a_1)$, which confirms that the corresponding policy is the safest one that always takes the action $a_1$, thus producing a deterministic discounted return.
In Figure \ref{fig:risky}, the risky fixed point satisfies $Q_1^{\text{risky}}(\cdot,a_2)=V_1^\pi(\cdot)$ and $Q_2^{\text{risky}}(\cdot,a_2)=V_2^\pi(\cdot)$, where $\pi$ is the policy from Example \ref{ex:a2} that always takes the riskiest action $a_2$.
Hence, our two control algorithms work as expected: they quickly converge to the desired fixed points, from which the safe or risky (deterministic) policies can be extracted.

\begin{figure}

\centering
\begin{subfigure}{.49\textwidth}
  \centering
  \includegraphics[width=1.0\linewidth]{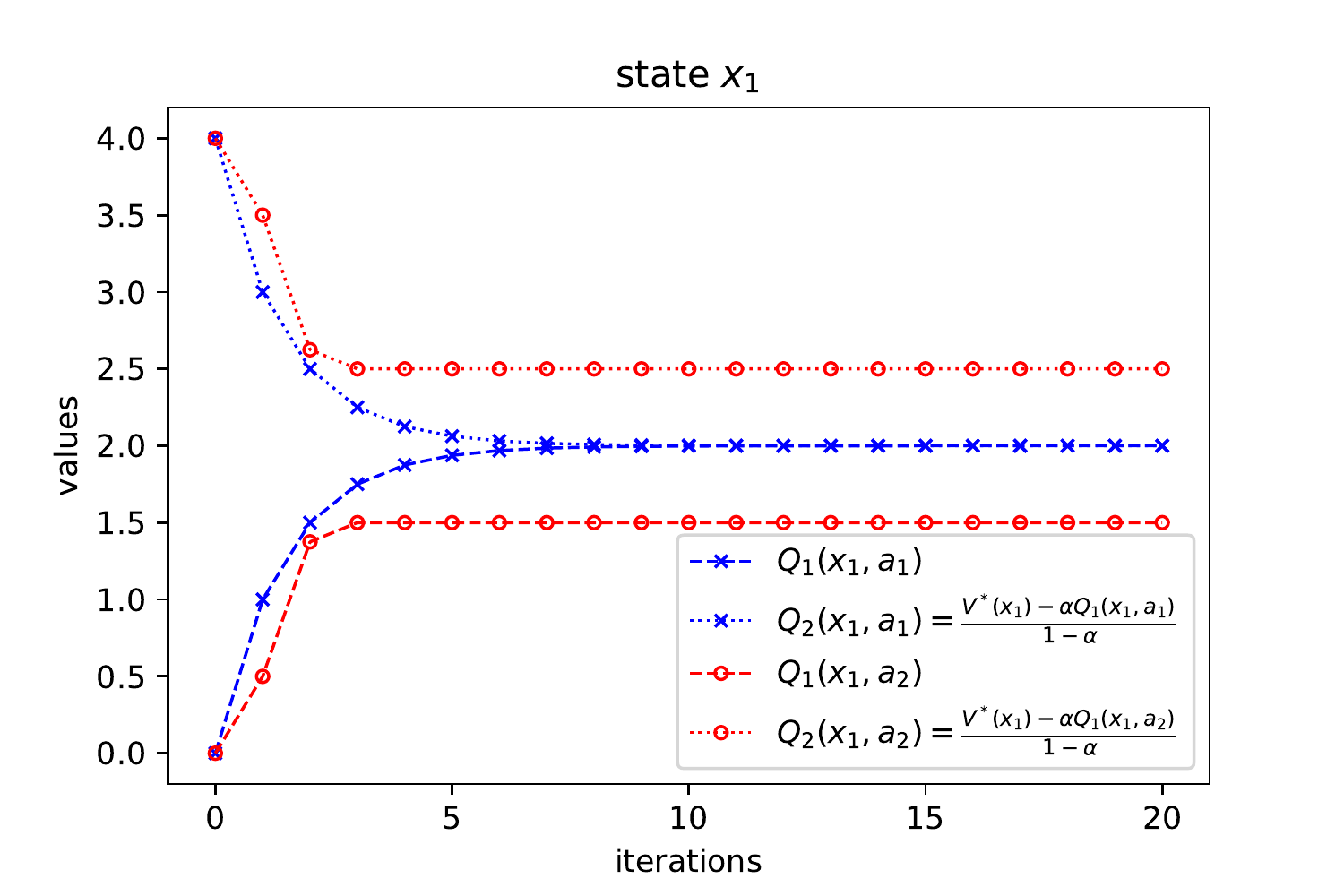}
\end{subfigure}
\begin{subfigure}{.49\textwidth}
  \centering
  \includegraphics[width=1.0\linewidth]{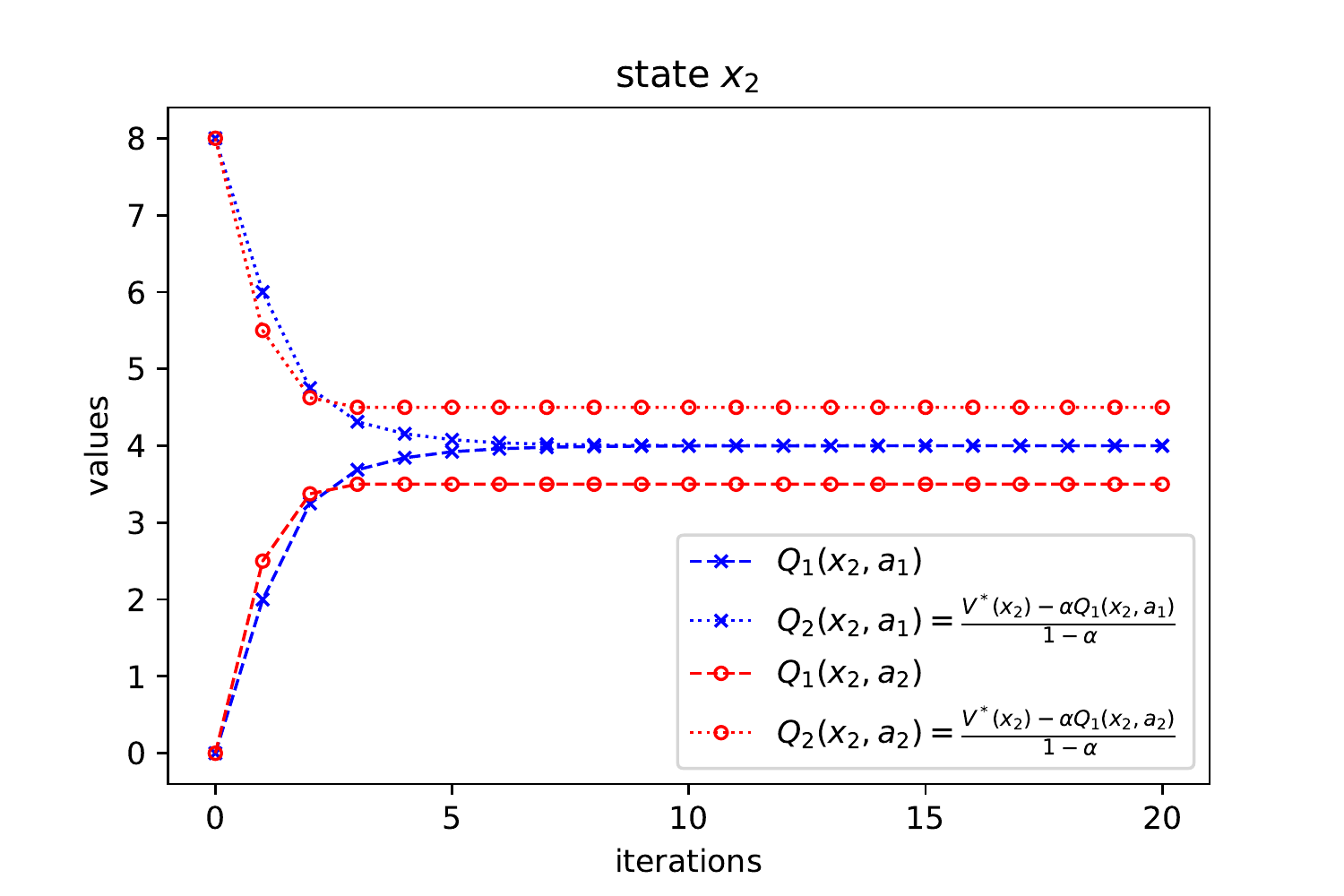}
\end{subfigure}

\caption{Safe control by successive iterations of the \textsc{Safe SVI} algorithm \ref{alg:safe_risky_svi}.}
\label{fig:safe}
\end{figure}

\begin{figure}

\centering
\begin{subfigure}{.49\textwidth}
  \centering
  \includegraphics[width=1.0\linewidth]{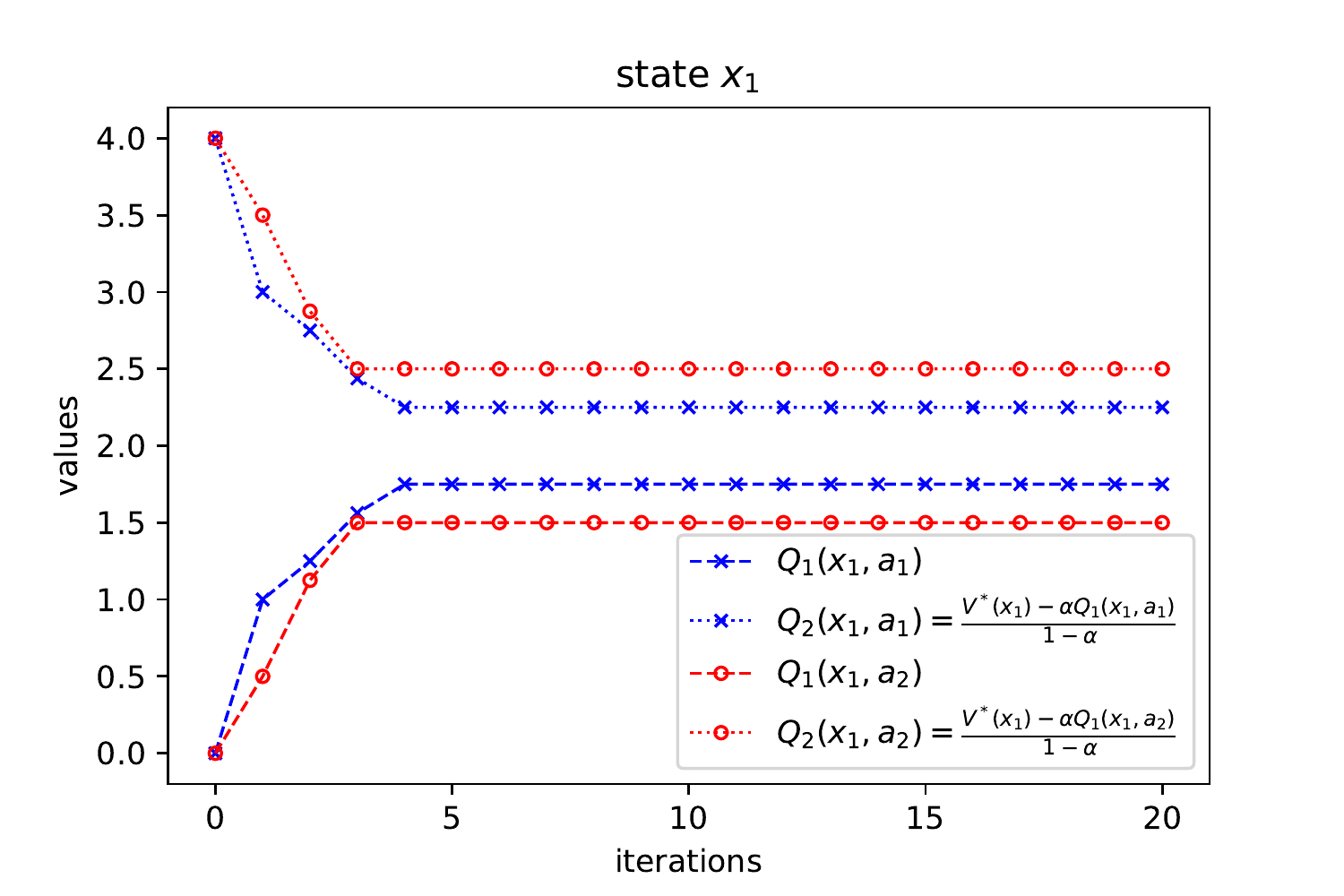}
\end{subfigure}
\begin{subfigure}{.49\textwidth}
  \centering
  \includegraphics[width=1.0\linewidth]{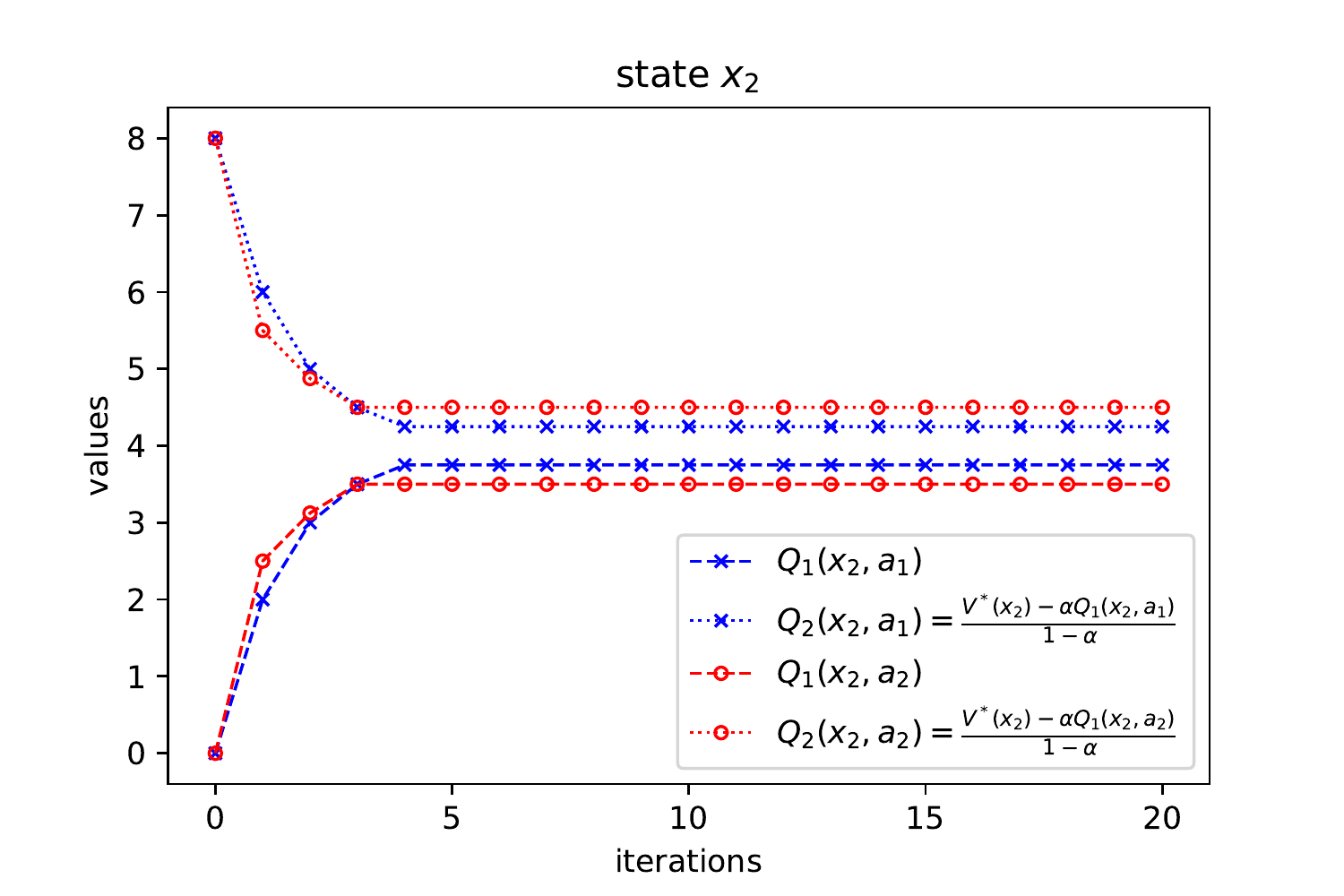}
\end{subfigure}

\caption{Risky control by successive iterations of the \textsc{Risky SVI} algorithm \ref{alg:safe_risky_svi}.}
\label{fig:risky}
\end{figure}

\section{Conclusion}
\label{sec:conclusion}

In this paper, we showed that the distributional perspective in MDPs can be leveraged to
define new robust control tasks in the exact tabular setting.
Our approach allows to distinguish safe from risky policies among the space of optimal policies.
In other words, we first require the classic control problem to be solved, so that all suboptimal actions
can be identified and removed from the action space, before we can apply our method.
This strong ``balanced MDP'' requirement constitutes the main limitation of our work.
Future lines of research include relaxing this assumption (e.g. in the safe case, by just requiring the condition $\argmax_a Q_1(x,a) = \argmin_a Q_2(x,a)$) or finding a natural class of MDPs with such structure.
Future work could as well investigate risk-seeking algorithms (based on the linear programming formulation provided in Appendix B) in a reinforcement learning setting with an agent only observing empirical transitions.




\newpage

\appendix
\section*{Appendix A. A technical prerequisite}
We recall the classic quantile representation of the expected value,
that is used several times throughout the paper.

\begin{lemma}{\textsc{(Expectation by quantiles)}.}
  \label{lem:quantile_expectation}
  Let $Z$ be a real-valued random variable with CDF $F$ and quantile function $F^{-1}$.
  Then,
  \begin{equation*}
    \mathbb{E}[Z] = \int_{\tau=0}^1 F^{-1}(\tau)d\tau .
  \end{equation*}
\end{lemma}

\begin{proof}
As any CDF is non-decreasing and right continuous (see, e.g., \citealp{billingsley2013convergence}), we have for all
$(\tau,z)\in(0,1)\times\mathbb{R}$:
\begin{equation*}
F^{-1}(\tau) \le z \Longleftrightarrow \tau \le F(z) .
\end{equation*}
Then, denoting by $U$ a uniformly distributed random variable over $[0, 1]$,
\begin{equation*}
  \mathbb{P}\{ F^{-1}(U) \le z \} = \mathbb{P}\{ U \le F(z) \} = F(z) ,
\end{equation*}
which shows that the random variable $F^{-1}(U)$ has the same distribution as $Z$.
Hence,
\begin{equation*}
  \mathbb{E}[Z] = \mathbb{E}[F^{-1}(U)] = \int_{\tau=0}^1 F^{-1}(\tau) d\tau .
\end{equation*}
\end{proof}


\section*{Appendix B. A linear programming formulation of risky control}
\label{appendix_LP}

Here, we show that the risky control task in a balanced MDP can be written as a linear program (LP).
From Lemma \ref{lem:avar}, we can get a closed-form expression of the left AVaR
appearing in the fixed point equation of the risky Bellman operator.
Given $(x,a)$, we need to sort the $|\mathbf{X}|=2|\pazocal{X}|$ particles
\begin{equation*}
  \left( r(x,a,x') + \gamma \min_{a'} Q_1^{\text{risky}}(x',a') \, , \, r(x,a,x') + \gamma \max_{a'} Q_2^{\text{risky}}(x',a') \right)_{x'\in\pazocal{X}} .
\end{equation*}
As $Q_1^{\text{risky}}\le Q_2^{\text{risky}}$, we restrict our attention to the following constrained permutations and transition kernels.

\begin{definition}{\textsc{(Constrained permutations)}.}
Let us denote by $\mathfrak{S}(\mathbf{X})$ the set of permutations $ \sigma : \mathbf{X} \rightarrow \{1, ..., |\mathbf{X}| \}$ that verify $\sigma(\underline{x}) < \sigma(\overline{x})$ for all $x\in\pazocal{X}$.
\end{definition}

It can be shown by induction on $|\pazocal{X}|=\frac{1}{2}|\mathbf{X}|$ that the cardinality of the set $\mathfrak{S}(\mathbf{X})$ is
\begin{equation*}
  |\mathfrak{S}(\mathbf{X})| = \frac{|\mathbf{X}| !}{2^{|\pazocal{X}|}} \, .
\end{equation*}
We refer to \cite{achab2019dimensionality} (and references therein) for a related analysis of constrained permutations in a statistical learning context.
With a slight abuse of notation, we denote for all $x,a,x'$,
\begin{equation*}
  P(\underline{x'}|x,a) := \alpha P(x'|x,a) \quad \text{ and } \quad P(\overline{x'}|x,a) := (1-\alpha) P(x'|x,a) \ .
\end{equation*}

\begin{definition}{\textsc{(Permutation kernels)}.}
  \label{def:Psigma}
For any constrained permutation $\sigma \in \mathfrak{S}(\mathbf{X})$, we define the transition kernel $\boldsymbol{P}_\sigma \in \Upsilon_{\alpha}$ as follows: for all $(x,a)\in \pazocal{X}\times\pazocal{A}, s'\in\mathbf{X}$,
\begin{multline*}
\boldsymbol{P}_\sigma(s'|\underline{x}, a)
= \frac{1}{\alpha} \max\left( 0 \ ,\ \min\left( P(s'|x,a) \ ,\ \alpha - \sum_{ s'': \sigma(s'')\le \sigma(s')-1} P(s''|x,a) \right) \right) \, ,\\
\text{ and } \quad \boldsymbol{P}_\sigma(s'|\overline{x}, a)
= \frac{1}{1-\alpha} \max\left( 0 \ ,\ \min\left( P(s'|x,a) \ ,\ \sum_{ s'': \sigma(s'')\le \sigma(s')} P(s''|x,a) - \alpha \right) \right) \, .
\end{multline*}
\end{definition}

\noindent \textbf{Primal \& dual LPs.}
For an initial state distribution $\nu_0=(\nu_0(x))_{x}>0$ over $\pazocal{X}$, we define the primal problem as:

\begin{multline}
\label{eq:primal_risky_permut}
\maximize_{V_1, \mathcal{V}} (1-\gamma) \langle \nu_0, V_1 \rangle\\
\text{subject to: } \quad \forall x\in\pazocal{X}, \forall a\in \pazocal{A}^*(x), \forall \sigma\in \mathfrak{S}(\mathbf{X}),\\
V_1(x) \le \sum_{s'\in\mathbf{X}} \boldsymbol{P}_\sigma(s'|\underline{x}, a) \left( r(x,a,s') + \gamma \mathcal{V}(s') \right) \ , \\
\mathcal{V}(\underline{x}) = V_1(x) \ , \quad \text{ and } \quad
\mathcal{V}(\overline{x}) = \frac{V^*(x) - \alpha V_1(x)}{1-\alpha} \ .
\end{multline}

\begin{lemma}
\label{lem:dual_risky_permut}
\begin{enumerate}[(i)]
\item The unique solution $(V_1^{\text{risky}}, \mathcal{V}^{\text{risky}})$ of the primal problem (\ref{eq:primal_risky_permut}) is given by:
\begin{equation*}
  V_1^{\text{risky}}(x) = \min_{a} Q_1^{\text{risky}}(x,a) \, .
\end{equation*}
\item The dual of problem (\ref{eq:primal_risky_permut}) is:
\begin{multline}
\minimize_{p \ge 0} \sum_{x,a,\sigma} p(x,a,\sigma) \cdot \left( \sum_{s'\in\mathbf{X}} \boldsymbol{P}_\sigma(s'|\underline{x}, a) r(x,a,s') + \frac{\gamma}{1-\alpha} \sum_{x'\in\pazocal{X}} \boldsymbol{P}_\sigma(\overline{x'}|\underline{x},a) V^*(x') \right) \\
\text{subject to: } \quad \forall x\in\pazocal{X}, \\
\sum_{a,\sigma} p(x,a,\sigma) = (1-\gamma) \nu_0(x) + \gamma \sum_{x',a',\sigma} \left( \boldsymbol{P}_\sigma(\underline{x}|\underline{x'},a') - \frac{\alpha}{1-\alpha} \boldsymbol{P}_\sigma(\overline{x}|\underline{x'},a') \right) p(x',a',\sigma) \ .
\end{multline}
\item Strong duality holds.
\end{enumerate}
\end{lemma}

\begin{proof}

\noindent (i) Primal LP.
This follows from the properties of the risky Bellman operator (Proposition \ref{prop:properties_risk}) combined with Lemma 2 in \cite{nilim2005robust}.

\noindent (ii) Dual LP.
The Lagrangian function of the LP (\ref{eq:primal_risky_permut}) is:
\begin{multline}
L(V_1, \mathcal{V}, p, \lambda_1, \lambda_2) = - (1-\gamma) \langle \nu_0, V_1 \rangle
+ \sum_{x,a,\sigma} p(x,a,\sigma) \left[ V_1(x) - \sum_{s'\in\mathbf{X}} \boldsymbol{P}_\sigma(s'|\underline{x}, a) \left( r(x,a,s') + \gamma \mathcal{V}(s') \right) \right] \\
+ \sum_{x} \lambda_1(x) \left[ \mathcal{V}(\underline{x}) - V_1(x) \right] + \sum_{x} \lambda_2(x) \left[ \mathcal{V}(\overline{x}) - \frac{V^*(x) - \alpha V_1(x)}{1-\alpha} \right] \\
= - \sum_{x,a,\sigma} p(x,a,\sigma) \sum_{s'} \boldsymbol{P}_\sigma(s'|\underline{x}, a) r(x,a,s') - \left\langle \lambda_2, \frac{V^*}{1-\alpha} \right\rangle \\
+ \sum_x V_1(x) \left[ - (1-\gamma) \nu_0(x) + \sum_{a,\sigma} p(x,a,\sigma) - \lambda_1(x) + \frac{\alpha}{1-\alpha} \lambda_2(x) \right] \\
+ \sum_x \mathcal{V}(\underline{x}) \left[ - \gamma \sum_{x',a',\sigma} \boldsymbol{P}_\sigma(\underline{x}|\underline{x'},a') p(x',a',\sigma) + \lambda_1(x) \right] \\
+ \sum_x \mathcal{V}(\overline{x}) \left[ - \gamma \sum_{x',a',\sigma} \boldsymbol{P}_\sigma(\overline{x}|\underline{x'},a') p(x',a',\sigma) + \lambda_2(x) \right] \ .
\end{multline}
The result follows by setting the gradients of $L$ with respect to $V_1$ and $\mathcal{V}$ to zero.

\noindent (iii) Strong duality.
It can be proved by Slater's condition, see \cite{boyd2004convex}.
\end{proof}





\vskip 0.2in
\bibliography{biblio_BAVaR}

\end{document}